%% file: acl_latex.tex
\documentclass[11pt]{article}

\usepackage[preprint]{acl}

\usepackage{times}
\usepackage{latexsym}
\usepackage[T1]{fontenc}
\usepackage[utf8]{inputenc}
\usepackage{microtype}
\usepackage{upquote}
\usepackage{graphicx}
\usepackage{xcolor}
\usepackage{subcaption}
\usepackage{booktabs}
\usepackage{tabularx}
\usepackage{array}
\usepackage{amsmath}
\usepackage{amssymb}
\usepackage{mathtools}
\usepackage{amsthm}
\usepackage{algorithm}
\usepackage{algorithmic}
\usepackage[capitalize,noabbrev]{cleveref}
\usepackage{placeins}
\usepackage{balance}

\graphicspath{{figures/}}

\newcolumntype{Y}{>{\centering\arraybackslash}X}

\definecolor{rawbg}{RGB}{250,232,232}
\definecolor{maskbg}{RGB}{250,244,215}
\definecolor{disanbg}{RGB}{229,244,236}
\definecolor{boxline}{RGB}{185,185,185}
\definecolor{brandblue}{RGB}{0,70,130}

\newcommand{\paperheader}{%
\noindent\makebox[\textwidth][c]{%
\begin{minipage}[b]{0.62\textwidth}
  \raggedright
  \raisebox{-0.28\height}{\includegraphics[height=0.36in]{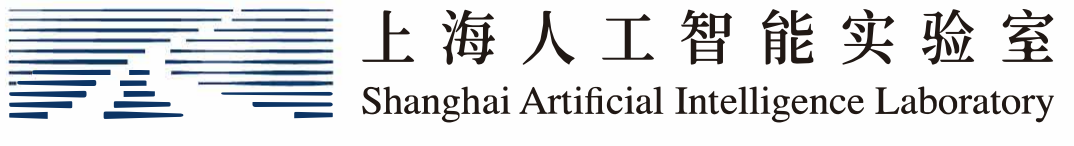}}%
  \hspace{0.12in}%
  \raisebox{-0.28\height}{\includegraphics[height=0.48in]{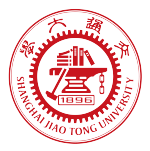}}%
\end{minipage}%
\begin{minipage}[b]{0.38\textwidth}
  \raggedleft
  \raisebox{-0.20\height}{\includegraphics[height=0.30in]{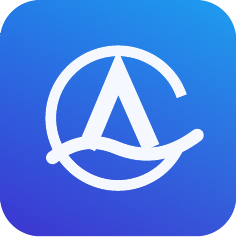}}\hspace{0.45em}%
  {\fontsize{20}{21}\selectfont\bfseries\color{brandblue} Intern-Shannon}
\end{minipage}}\\[-0.35em]
{\color{black}\rule{\textwidth}{0.8pt}}%
}

\theoremstyle{plain}
\newtheorem{theorem}{Theorem}[section]

\newtheorem{lemma}[theorem]{Lemma}

\theoremstyle{definition}
\newtheorem{definition}[theorem]{Definition}
\newtheorem{assumption}[theorem]{Assumption}

\theoremstyle{remark}

\makeatletter
\renewenvironment{abstract}%
  {{\centering\large\textbf{\abstractname}\par}%
    \vspace{0.2em}%
    \begin{list}{}%
      {\setlength{\rightmargin}{0.6cm}%
        \setlength{\leftmargin}{0.6cm}%
        \setlength{\topsep}{0pt}%
        \setlength{\partopsep}{0pt}%
        \setlength{\parsep}{0pt}%
        \setlength{\itemsep}{0pt}}%
      \item[]\ignorespaces%
      \@setsize\normalsize{12pt}\xpt\@xpt
  }%
  {\unskip\end{list}}
\makeatother

\title{\paperheader\\[0.5in]
Privacy-Preserving Text Sanitization for Distributed Agents Collaboration via Disentangled Representations}

\author{
\textbf{Xuan Liu}\textsuperscript{*} \quad
\textbf{Hefeng Zhou}\textsuperscript{*} \quad
\textbf{Sicheng Chen} \quad
\textbf{Chao Yang} \quad
\textbf{Xingcheng Xu} \\
\textbf{Jingjing Qu}\textsuperscript{†} \quad
\textbf{Jiong Lou} \quad
\textbf{Jie LI} \quad
\textbf{Xia Hu} \\
{\normalfont\textsuperscript{1}Shanghai Artificial Intelligence Laboratory} \\
{\normalfont\textsuperscript{2}Shanghai Jiao Tong University}
}

\begin{document}
\maketitle
\begingroup
\renewcommand{\thefootnote}{*}
\footnotetext{Equal contribution. †Corresponding author. Code is available at \url{https://github.com/RezinChow/DiSan}. Intern-Shannon is the next-generation Agentic Operating System developed by Shanghai AI Lab, which will be officially released soon. }
\endgroup

\begin{abstract}
\input{sections/abstract}
\end{abstract}

\input{sections/introduction}
\input{sections/related_work}
\input{sections/problem_statement}
\input{sections/method}
\input{sections/experiments}
\input{sections/conclusion}

\input{sections/limitations}

\input{acl_latex.bbl}
\clearpage
\appendix
\onecolumn
\input{sections/appendix_threat_model}
\input{sections/appendix_method_details}
\input{sections/appendix_experimental_supplement}
\input{sections/appendix_theoretical_analysis}

\end{document}

%% file: sections/abstract.tex
When distributed agents exchange text across organizational boundaries, privacy leakage arises not only from explicit identifiers but also from \emph{distributional signatures} such as formatting conventions, vocabulary choices, and syntactic patterns. We propose \textit{DiSan} (\textbf{Di}sentangled \textbf{San}itization), a privacy-preserving sanitization framework and a built-in component of Intern-Shannon for multi-agent collaboration. \textit{DiSan} uses a two-stream encoder to factorize text into a source-invariant role subspace that preserves task semantics and a source-identifying style subspace that remains local. Federated prototype alignment and adversarial regularization enable joint training without centralizing raw text. Experiments show that identifier-level masking is insufficient: masking 19.2\% of tokens reduces TF-IDF stylometric attribution by only 18.6\%. By contrast, \textit{DiSan} reduces answer-level PII exposure by 20$\times$ while maintaining 83\% answer faithfulness on a distributed multi-agent RAG benchmark, and lowers Enron stylometric attribution by 73.2\% under TF-IDF and 70.6\% under a neural probe.

%% file: sections/introduction.tex
\section{Introduction}
\begin{figure}[!tb]
\vskip -0.05in
\begin{center}
\centerline{\includegraphics[width=1\columnwidth]{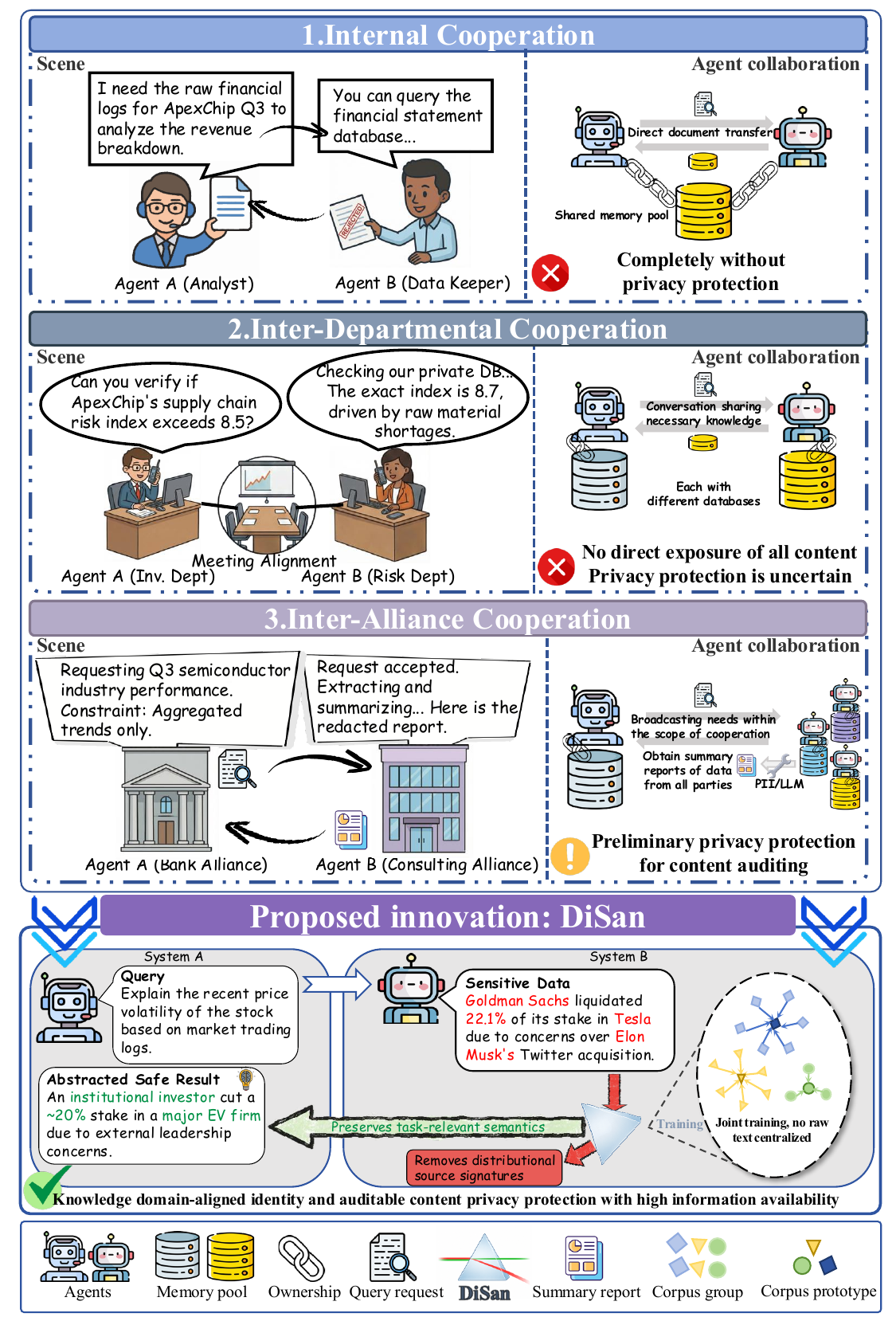}}
\caption{Privacy risks in cross-organizational text sharing. \textbf{Top}: Three representative collaboration settings, ranging from no protection to preliminary filtering, each leaving source-identifying content exposed. \textbf{Bottom}: \textsc{DiSan} produces sanitized text that preserves task-relevant semantics while removing both explicit PII and distributional source signatures.}
\label{fig:motivation}
\end{center}
\vskip -0.2in
\end{figure}

Cross-organizational collaboration on text-intensive tasks, including retrieval-augmented generation~\cite{NEURIPS2020_RAG}, distributed question answering, and cross-institutional case retrieval~\cite{stubbs2015automated}, requires parties to share textual evidence while keeping raw data local. Each party may hold proprietary document collections with distinct domain expertise, and a common pattern is for a requesting party to seek auxiliary evidence from helpers, who retrieve local snippets and transmit them for downstream use such as answer synthesis~\cite{minaee2024large_llmsurvey}.

However, any inter-agent text exchange exposes private organizational information beyond what public capability tags disclose. As illustrated in \cref{fig:motivation}, this risk spans all levels of cross-organizational collaboration: from intra-organization transfers with no protection to inter-alliance sharing where only preliminary filtering is applied. In each setting, shared text leaks private organizational information at two levels: explicitly through identifiers such as names, account numbers, and addresses, and implicitly through \emph{distributional signatures}: formatting conventions, vocabulary choices, and syntactic patterns that encode the originating party's internal practices~\cite{Malik2011EnhancedSA}. This is fundamentally a representation-level problem, not an identifier-level one: private organizational information is a property of the text distribution, not of individual identifiers, and anonymization methods that operate in text space cannot alter these distributional properties. \Cref{tab:privacy_comparison} makes this concrete: raw sharing exposes not only a counterparty identity but also a proprietary bulletin format, reference scheme, and sector taxonomy; placeholder masking hides surface identifiers but leaves the naming convention intact and collapses distinct entities into generic tokens, weakening grounding, provenance tracking, and cross-document aggregation. \textsc{DiSan} instead preserves role facts such as exposure amount, sector, rating change, and review status while suppressing source-specific fingerprints.

\input{sections/privacy_comparison_table}

Existing approaches address symptoms rather than structure. Rule-based PII detectors~\cite{Li2021FederatedLO} target individual identifiers such as named entities and account numbers but are blind to distributional signatures, since private organizational information is distributed across the text as statistical patterns, not localized to individual spans. LLM-based paraphrasing~\cite{xiao2024large} reshuffles surface form but provides no mechanism to ensure the output distribution is statistically source-invariant. Federated learning~\cite{fl-mcmahan17a} decentralizes model training but produces shared \emph{predictors}, not shareable \emph{data}. The core challenge is structural: what is the minimal sufficient representation that preserves task semantics while being statistically source-invariant?

We propose \textsc{DiSan}, a sanitization framework for cross-agent evidence exchange.
It learns a role--style factorization of each evidence snippet, where the \emph{role} subspace preserves task-relevant semantics and the \emph{style} subspace captures source-linked variation.
Orthogonality encourages the two subspaces to separate, while prototype alignment keeps role representations comparable across non-IID agents without centralizing raw text.
The resulting sanitizer produces shareable text from the role stream while keeping style information local. \textsc{DiSan} further serves as a key privacy-preserving component of \textbf{\textit{Intern-Shannon}}, where it is integrated as a built-in text-sanitization module and can be invoked on demand during multi-agent collaboration.

\paragraph{Contributions.}
\textbf{(i) Disentangled sanitization for text sharing:} We formulate cross-agent text sanitization as role--style factorization, separating task semantics from source-linked variation.
\textbf{(ii) Federated role alignment:} We introduce lightweight prototype alignment to stabilize role spaces across non-IID agents without centralizing raw text.
\textbf{(iii) Privacy diagnostics across sharing surfaces:} We evaluate privacy at the output, representation, and prototype levels, distinguishing application-stage leakage from training-stage artifacts.
\textbf{(iv) Empirical validation:} On distributed-agent RAG, \textsc{DiSan} reduces answer-level PII exposure by 20$\times$ while maintaining 83\% answer faithfulness. On Enron emails, it reduces TF-IDF stylometric attribution by 73.2\%, substantially outperforming identifier-level masking.

\FloatBarrier

%% file: sections/privacy_comparison_table.tex
\begin{table}[!t]
\centering
\footnotesize
\setlength{\fboxsep}{4pt}
\setlength{\fboxrule}{0.4pt}
\caption{Compact CorporateBank$\to$AssetManager financial-risk example. \textbf{Bold} marks source-identifying patterns that survive identifier masking.}
\label{tab:privacy_comparison}
\makebox[\columnwidth][c]{\fcolorbox{boxline}{rawbg}{\begin{minipage}[t]{0.955\columnwidth}\raggedright
\textbf{Raw private text ($d$).}
\textit{``Per \textbf{Meridian Bank's Counterparty Risk Bulletin (Ref: CR-2024-047)}, \textbf{Apex Dynamics} carried \$6.1M exposure as of Q3 close, was downgraded to BB+, and was flagged for portfolio review.''}
\end{minipage}}}

\makebox[\columnwidth][c]{\fcolorbox{boxline}{maskbg}{\begin{minipage}[t]{0.955\columnwidth}\raggedright
\textbf{Placeholder masking only.}
\textit{``Per \textbf{Meridian Bank's Counterparty Risk Bulletin (Ref: [ID])}, \textbf{[ORG]} carried \$6.1M exposure as of Q3 close, was downgraded to BB+, and was flagged for portfolio review.''}
\end{minipage}}}

\makebox[\columnwidth][c]{\fcolorbox{boxline}{disanbg}{\begin{minipage}[t]{0.955\columnwidth}\raggedright
\textbf{\textsc{DiSan} output ($\tilde{d}$).}
\textit{``A corporate-bank Q3 risk bulletin flags an industrials counterparty with \$6.1M exposure, a BB+ downgrade, and mandatory portfolio review.''}
\end{minipage}}}
\end{table}

%% file: sections/related_work.tex
\section{Related Work}
\paragraph{Privacy-Preserving Machine Learning.}
Protecting privacy across distributed data sources is a persistent challenge in collaborative machine learning~\cite{Li2021FederatedLO}. Differential privacy (DP)~\cite{DP} provides formal guarantees, and DP-SGD~\cite{DP-SGD} extends these to deep learning~\cite{DP_convex_optimization,DeepLearning_DP,Discrete_Gaussian_Differential_Privacy}. Federated learning~\cite{fl-mcmahan17a} enables collaborative model training without sharing raw data~\cite{FL_review,horizonFL,SCAFFOLD}. Within this paradigm, prototype-based methods~\cite{Tan2021FedProtoFP,FedTGP_2024}, representation-based approaches~\cite{MOON,FedCon}, and communication-efficient techniques~\cite{oneshotFL,fedllm_finetune} have been proposed. Recent federated RAG formulations aim to enable multi-party retrieval under privacy constraints~\cite{qian2025hyfedrag,he2025pfedrag,mao2025privacy,chakraborty2025federatedretrievalaugmentedgenerationsystematic}. While these methods focus on training shared models or retrievers, our work addresses a complementary problem: \emph{sanitizing text data itself} so it can be safely shared for downstream use.

\paragraph{Text Sanitization and De-identification.}
Traditional text de-identification relies on rule-based or NER-based PII detection followed by masking or replacement~\cite{Malik2011EnhancedSA}. While effective for explicit identifiers, these methods miss implicit leakage through writing style, document structure, and domain-specific patterns. Authorship attribution research~\cite{stamatatos2009survey} demonstrates that stylometric features can identify sources even from short texts. Recent work explores LLM-based paraphrasing for privacy~\cite{shi2025privacy} and DP-based text generation~\cite{meisenbacher2024dpllm,xie2024dptext}. However, DP-text methods incur severe utility degradation (30--50\% coherence loss for $\epsilon < 10$) that is prohibitive for RAG applications requiring semantic fidelity~\cite{meisenbacher2024dpllm}. Our approach addresses both explicit PII and implicit stylistic fingerprints through learned disentanglement, achieving strong empirical privacy without the utility cost of DP noise on text outputs.

\paragraph{Disentangled Representations.}
Disentangled representation learning aims to separate independent factors of variation in data~\cite{bengio2013representation}. In NLP, disentanglement has been applied to separate content from style for style transfer~\cite{john2019disentangled}, sentiment from semantics, and speaker identity from linguistic content in speech~\cite{qian2019autovc}. Recent work enables explicit control in text generation via disentangled representations~\cite{liu2024multi,han2024disentangled}. These works establish that separating semantic content from stylistic or identity-related factors supports both controllable generation and privacy; we apply the same principle to text sharing, where the goal is to isolate task-relevant content from source-identifying patterns.

%% file: sections/problem_statement.tex
\section{Problem Statement}
\label{sec:prelim}

\subsection{Text Sharing for Distributed Agent Collaboration}
\label{sec:routing}
We consider $C$ distributed agents, each hosting a private document repository $\mathcal{D}_c$.
When a requesting agent cannot answer a query $q$ locally, it routes to a candidate helper set $\mathcal{C}(q)$ via embedding similarity over public \textbf{capability tags}, such as ``AssetManager'' and ``CorporateBank''.
Each helper retrieves a local snippet $d$ and returns a \emph{sanitized} snippet $\tilde d$ for downstream use in RAG or distributed question answering.

\cref{tab:privacy_comparison} concretizes the sanitization goal: raw text exposes explicit PII and institutional style fingerprints; placeholder masking removes identifiers but leaves distributional signatures intact and can weaken downstream grounding when masked spans are task-relevant; only a representation-level approach targets both privacy risks while preserving task semantics.

\subsection{Threat Scope}
\label{sec:threat}

Our primary privacy concern arises in the \emph{application stage}. A requesting agent receives sanitized evidence $\tilde d$ from a helper agent and may try to infer information beyond the helper's public capability tag, including explicit PII, writing patterns tied to the source, or organizational document fingerprints.
The sanitizer is trained federatively across the same agents; the coordinator follows the protocol and does not access raw text.
We additionally evaluate whether training artifacts such as uploaded prototypes contain persistent distributional signatures tied to individual sources beyond public tags.
We do not claim formal differential privacy or robustness to malicious servers, poisoning, prompt injection, or collusion; the complete threat model is in \cref{app:threat}.

\paragraph{Goals and validation.}
We do not hide public capability tags or agent participation; we aim to prevent leakage beyond these facts while preserving downstream RAG utility.
Our validation follows this scope: PII leakage and answer exposure measure explicit identifiers, stylometry measures source signatures in transmitted text, embedding and prototype attribution diagnose learned sharing artifacts, and F1, faithfulness, and ChunkHit@3 measure utility.

%% file: sections/method.tex
\section{Method}
\label{sec:method}

\paragraph{Overview.}
\textsc{DiSan} enforces role--style orthogonality as an explicit architectural constraint: a two-stream encoder projects each input into a \textbf{role} subspace encoding source-invariant task semantics and a \textbf{style} subspace encoding agent-specific variation. Role representations are used to decode sanitized text $\tilde{d}$; style representations assist local generation for fluency and are then discarded. Only $\tilde{d}$ crosses the privacy boundary.

Training without centralizing raw text poses a calibration challenge: per-agent isolation causes role spaces to drift across agents, degrading both utility and privacy. \textsc{DiSan} addresses this by exchanging compact role prototypes, aligning local role distributions to shared global anchors, and applying adversarial regularization to suppress source-specific prototype signatures beyond public agent tags. \cref{fig:method} illustrates the architecture.
\FloatBarrier
\begin{figure*}[!tb]
\vskip 0.1in
\begin{center}
\centerline{\includegraphics[width=0.78\textwidth]{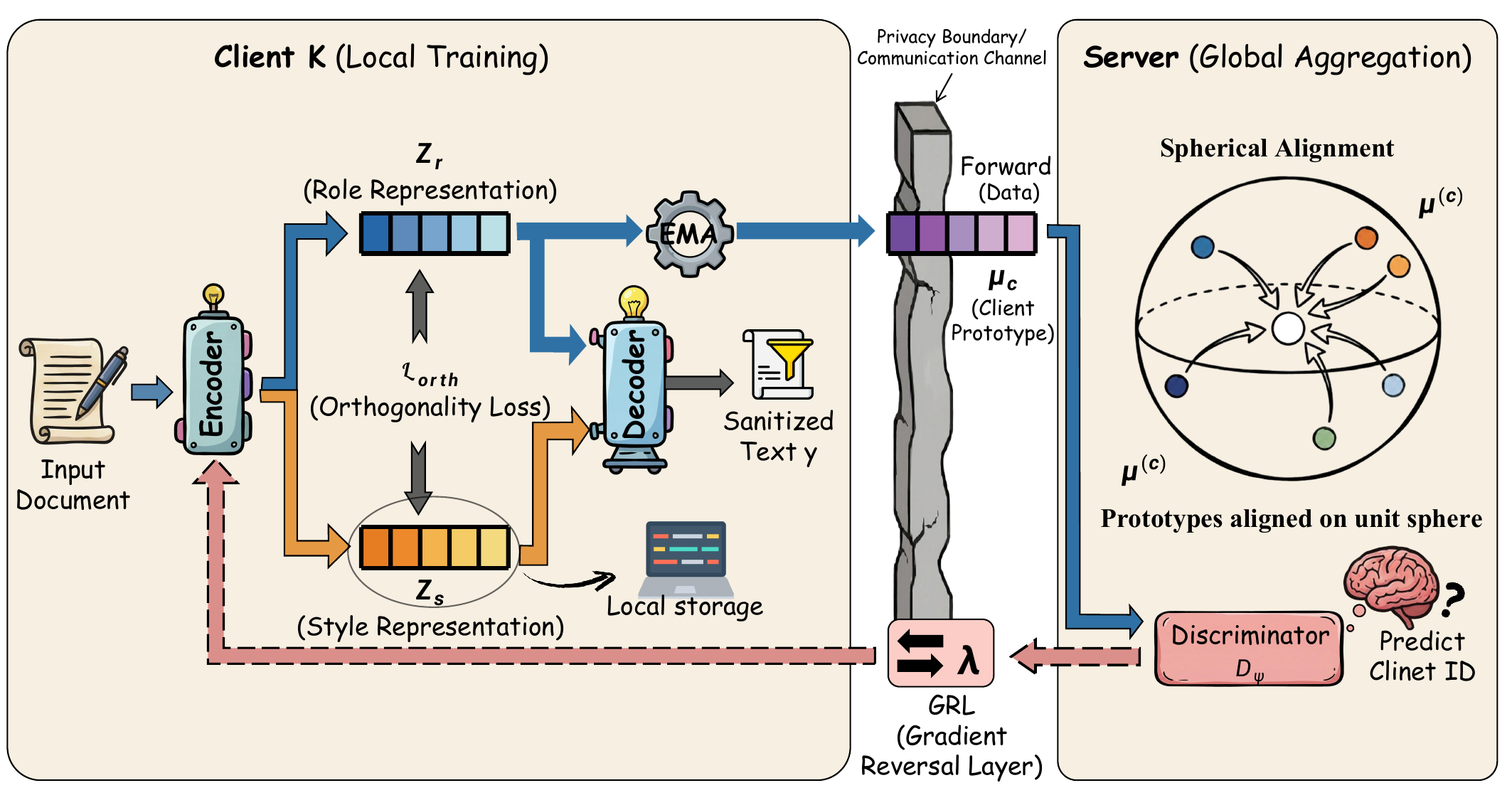}}
\caption{\textsc{DiSan} architecture. \textbf{Left (agent):} A two-stream encoder produces role representations \(\mathbf{Z}_r\) capturing source-invariant semantics and style representations \(\mathbf{Z}_s\) capturing agent-specific variation; \(\mathcal{L}_{\text{orth}}\) enforces their separation. Both are fused locally for decoding; \(\mathbf{Z}_s\) is discarded after use and never transmitted. Only role prototypes \(\boldsymbol{\mu}_c\) cross the privacy boundary. \textbf{Right (server):} Prototypes are aligned spherically; gradient reversal suppresses source-specific signatures in uploaded prototypes.}
\label{fig:method}
\end{center}
\vskip -0.2in
\end{figure*}

\subsection{Role--Style Disentangled Encoder}
\label{sec:two_stream}

\paragraph{Two-stream projection.}
Given an evidence sequence \(d=(d_1,\ldots,d_T)\), a pretrained encoder produces hidden states
\(\mathbf{H}=\mathrm{Encoder}(d)\in\mathbb{R}^{T\times d_{\text{enc}}}\).
We project \(\mathbf{H}\) into a \emph{role} stream and a \emph{style} stream (backbone and projection dimensions in \cref{app:arch_details}):
\begin{align}
\mathbf{Z}_r &= \mathbf{H}\mathbf{W}_r,\quad \mathbf{W}_r\in\mathbb{R}^{d_{\text{enc}}\times d_r},\ \mathbf{Z}_r\in\mathbb{R}^{T\times d_r}, \label{eq:role_proj}\\
\mathbf{Z}_s &= \mathbf{H}\mathbf{W}_s,\quad \mathbf{W}_s\in\mathbb{R}^{d_{\text{enc}}\times d_s},\ \mathbf{Z}_s\in\mathbb{R}^{T\times d_s}. \label{eq:style_proj}
\end{align}
In our implementation, the encoder--decoder backbone is LongT5-TGlobal-Base.
Both \(\mathbf{Z}_r\) and \(\mathbf{Z}_s\) are 256-dimensional token streams; their concatenation forms a 512-dimensional bottleneck that is projected back to \(d_{\text{enc}}\) before decoding.
Intuitively, \(\mathbf{Z}_r\) captures shareable content structure (entities/relations/events),
while \(\mathbf{Z}_s\) captures agent-specific phrasing and formatting.

\paragraph{Fusion for generation.}
We fuse the two streams before decoding via \(\mathbf{H}_{\text{fused}} = g([\mathbf{Z}_r; \mathbf{Z}_s])\), where \(g\) is a learned projection back to \(d_{\text{enc}}\).
An optional residual path \(\mathbf{H}_{\text{out}} = \alpha \cdot \mathbf{H}_{\text{fused}} + (1-\alpha)\cdot \mathbf{H}\) stabilizes early training (details in \cref{app:fusion}).
The pretrained LongT5 Transformer decoder then generates \(\tilde d\) autoregressively from the modified encoder states; \textsc{DiSan} inserts a role--style bottleneck rather than replacing the generator.

Although \(\mathbf{Z}_s\) may contain source-linked variation, it is used only inside the helper during decoding and is never transmitted.
The decoder is trained to preserve role facts while removing identifiers and source-specific wording, so the local style stream serves as a fluency aid rather than a shared artifact.
Residual leakage is evaluated on the transmitted \(\tilde d\) through output-level PII and stylometry metrics.

\paragraph{Disentanglement.}
Let \(\bar{\mathbf{z}}_r=\frac{1}{T}\sum_{t=1}^{T}\mathbf{z}_{r,t}\) and \(\bar{\mathbf{z}}_s=\frac{1}{T}\sum_{t=1}^{T}\mathbf{z}_{s,t}\) denote mean-pooled vectors.
We encourage separation via:
\begin{equation}
\mathcal{L}_{\text{orth}}
=\cos^2(\bar{\mathbf{z}}_r,\bar{\mathbf{z}}_s)
=\left(\frac{\bar{\mathbf{z}}_r^\top \bar{\mathbf{z}}_s}{\|\bar{\mathbf{z}}_r\|_2\,\|\bar{\mathbf{z}}_s\|_2 + \epsilon}\right)^2.
\label{eq:orth}
\end{equation}

\subsection{Prototype Alignment on the Role Space}
\label{sec:proto}

To train the privacy transformer under non-IID agents, we align \emph{role distributions} across agents using lightweight role prototypes, providing global anchors without sharing raw evidence.

\paragraph{Prototype computation and aggregation.}
For each role/placeholder type \(k\in\mathcal{K}\), each agent computes a local prototype \(\boldsymbol{\mu}^{(c)}_{k}\) as an EMA of batch role centroids (token-level averages over type-\(k\) positions).
At each training round, the server aggregates a sample-weighted global prototype \(\boldsymbol{\mu}^*_{k}\).
Full definitions are provided in \cref{app:proto_details}.

\paragraph{Spherical alignment.}
To avoid magnitude-based leakage, we align on the unit hypersphere:
\begin{equation}
\mathcal{L}_{\text{proto}}
= \sum_{k\in\mathcal{K}}
\big(1 - \cos(\hat{\bar{\mathbf{z}}}_{r,k}, \hat{\boldsymbol{\mu}}^*_{k})\big),
\label{eq:proto_loss}
\end{equation}
where \(\hat{\cdot}\) denotes \(\ell_2\)-normalization, and \(\bar{\mathbf{z}}_{r,k}\) is the batch role centroid for type \(k\), defined as the token-level average over positions labeled as type \(k\) (see \cref{app:proto_details}).

\paragraph{Prototype-level adversarial training.}
Prototypes may still carry source-specific distributional signatures beyond public agent tags.
We apply a discriminator \(D_\psi\) with gradient reversal (GRL)~\cite{raff2018gradient} to make prototypes less predictive of their source:
\begin{equation}
\mathcal{L}_{\text{adv}}
=\sum_{k\in\mathcal{K}}
\mathrm{CE}\!\left(D_\psi(\mathrm{GRL}_\gamma(\bar{\boldsymbol{\mu}}^{(c)}_{k})),\,c\right),
\label{eq:adv}
\end{equation}
where \(\gamma\) is the GRL strength and \(\bar{\boldsymbol{\mu}}^{(c)}_{k}\) is a gradient-enabled estimate (\cref{app:proto_details}).
All attribution discriminators use the same compact MLP template, input \(\to 128 \to 128 \to C\), with ReLU activations and dropout 0.1; the prototype discriminator takes 256-dimensional role prototypes as input.

\subsection{Federated Optimization}
\label{sec:training}

\paragraph{Local objective.}
On agent \(c\), we minimize:
\begin{equation}
\begin{split}
\mathcal{L}^{(c)}_{\text{local}}
={} & \mathcal{L}_{\text{seq}}
+\lambda_{\text{orth}}\mathcal{L}_{\text{orth}}
+\lambda_{\text{p}}\mathcal{L}_{\text{proto}} \\
& +\lambda_{\text{adv}}\mathcal{L}_{\text{adv}}
+\mathcal{L}_{\text{prox}},
\end{split}
\label{eq:total_loss}
\end{equation}
where \(\mathcal{L}_{\text{seq}}\) is the token-level cross-entropy on sanitization targets \(\tilde d\), and \(\mathcal{L}_{\text{prox}}=\frac{\nu}{2}\|\theta-\theta^*\|_2^2\) is a FedProx term discouraging drift from the global model (we use \(\nu\) to distinguish from prototype symbols \(\boldsymbol{\mu}\)).
After local optimization, the server aggregates model weights and prototypes.
Before uploading, agents apply \(\ell_2\)-normalization and Gaussian noise perturbation ($\sigma=0.01$) to prototypes (\cref{app:dp_proto}).
For communication efficiency, we train LoRA adapters on the encoder attention projections together with the role projection and fusion layers.
The style projection remains local and is never uploaded; per-round prototype exchange is only \(|\mathcal{K}|\times 256\) floating-point values per agent, negligible compared with model synchronization.

\begin{table*}[htbp!]
\vskip -0.05in
\caption{Main single-round results.
\textbf{Bold} = best, \underline{underline} = second best.
Metrics are defined in \cref{sec:metrics}.}
\label{tab:main_results}
\vskip 0.05in
\begin{center}
\resizebox{0.95\textwidth}{!}{%
\renewcommand{\arraystretch}{0.95}
\begin{tabular}{lcccccccc}
\toprule
Method
& \multicolumn{4}{c}{Semantic Similarity $\uparrow$}
& Faith. $\uparrow$
& ChunkHit@3 $\uparrow$
& \multicolumn{2}{c}{PII Exposure $\downarrow$} \\
\cmidrule(lr){2-5}\cmidrule(lr){8-9}
& F1 & Prec. & Rec. & Cos.
&  &  & Avg. PII & Ans. Rate \\
\midrule
RAG Baseline without PII process
& 0.6602 & 0.5519 & 0.8214 & 0.6421
& 86.10\% & 75.2\% & 6.4753 & 11.8\% \\
\midrule
Placeholder (\texttt{gliner-pii-large-v1.0})
& 0.5071 & 0.4624 & 0.5613 & 0.5923
& 74.60\% & 62.4\% & \underline{0.4714} & \underline{1.4\%} \\

Placeholder (\texttt{piiranha-v1-detect})
& 0.4894 & 0.4352 & 0.5587 & 0.5643
& 71.82\% & 58.6\% & 0.5126 & 1.6\% \\

Placeholder (\texttt{deberta-pii-finetuned})
& 0.4768 & 0.4541 & 0.5018 & 0.5819
& 72.63\% & 60.2\% & 0.6287 & 2.0\% \\
\midrule

LLM paraphrasing (\texttt{llama-3.1-8b-instruct})
& 0.4991 & 0.4427 & 0.5719 & 0.5872
& 77.89\% & 66.4\% & 1.2479 & 3.6\% \\

LLM paraphrasing (\texttt{Qwen2.5-7B-Instruct})
& \underline{0.5407} & \underline{0.5011} & 0.5872 & 0.6218
& \underline{79.35\%} & \underline{67.0\%} & 0.8412 & 2.6\% \\

LLM paraphrasing (\texttt{GLM-4-9B-0414})
& 0.4732 & 0.4315 & 0.5237 & 0.6072
& 77.13\% & 64.6\% & 2.1091 & 3.8\% \\

Policy gating
& 0.5013 & 0.3876 & \textbf{0.7936} & \underline{0.6322}
& 78.49\% & 65.2\% & 0.9714 & 2.2\% \\

\textbf{\textsc{DiSan} (ours)}
& \textbf{0.5631} & \textbf{0.5341} & \underline{0.6508} & \textbf{0.6558}
& \textbf{83.17\%} & \textbf{73.4\%} & \textbf{0.1337} & \textbf{0.6\%} \\
\bottomrule
\end{tabular}%
}
\end{center}
\vskip -0.05in
\end{table*}

\subsection{Deployment: Multi-Agent Text Sharing}
\label{sec:deployment}

The preceding sections describe \emph{how to train} the sanitizer; this section describes \emph{how agents use it} at deployment.
Each agent deploys the trained model as a local sanitizer: given raw text \(d\), it produces \(\tilde{d}\) via role--style fusion and decoding.
Crucially, \(\mathbf{Z}_s\) is used only locally to improve generation quality and is \textbf{discarded after decoding}. Only the sanitized text \(\tilde{d}\) is transmitted.
This design enables flexible inter-agent data-sharing protocols while preserving the privacy guarantees established during training.

\paragraph{RAG pipeline.}
A requesting agent routes a query \(q\) to helper agents \(\mathcal{C}\) via capability-based routing (\cref{sec:routing}).
Each helper \(c \in \mathcal{C}\) retrieves local evidence \(d_c\) and returns:
\begin{equation}
\tilde{d}_c = \mathrm{Sanitize}_\theta(d_c, q),
\end{equation}
where \(\mathrm{Sanitize}_\theta\) denotes the trained sanitizer (\cref{sec:two_stream}--\ref{sec:training}).
The requester aggregates \(\tilde{\mathcal{D}} = \{\tilde{d}_c \mid c \in \mathcal{C}\}\) and generates a final answer.
This protocol naturally extends to multi-turn settings where the requester iteratively refines queries based on accumulated evidence (see \cref{sec:mas_analysis} for details).
Convergence analysis is provided in \cref{app:theory_convergence}.

%% file: sections/experiments.tex
\section{Experiments}
\label{sec:experiments}

\newcommand{\datasetname}{\texttt{synthetic\_\allowbreak{}pii\_\allowbreak{}finance\_\allowbreak{}multilingual}}
\newcommand{\glinermodel}{\texttt{gliner-\allowbreak{}pii-\allowbreak{}large-\allowbreak{}v1.0}}
\newcommand{\piiranhamodel}{\texttt{piiranha-\allowbreak{}v1-\allowbreak{}detect-\allowbreak{}personal-\allowbreak{}information}}
\newcommand{\debertapiimodel}{\texttt{deberta-\allowbreak{}pii-\allowbreak{}finetuned}}

\subsection{Experimental Setup}
\label{sec:exp_setup}

\paragraph{Dataset.}
We use a multilingual synthetic finance corpus with annotated PII spans, \datasetname~\cite{gretel-synthetic-pii-finance-multilingual-2024}.

\paragraph{Agent configuration.}
To simulate distributed agent collaboration, we construct \(C{=}7\) agents by assigning each agent a disjoint inventory of document types.
The resulting partition induces non-IID skews by design.
We denote agents by their capability tags: CorporateBank, AssetManager, FinTechPay, CorpGroup, MarketForecaster, ComplianceConsult, and SupplierCo.
These tags are treated as public in our threat model (\cref{sec:threat}); we therefore focus on leakage \emph{beyond} this public prior.
The exact doc-type identifiers used for each agent are listed in \cref{app:doc_types}.

\paragraph{RAG evaluation pipeline.}
Documents are chunked into fixed windows (256 tokens, overlap 50) and indexed per agent.
Given a query, the requesting agent routes to a candidate helper set via tag-based routing (\cref{sec:routing}), retrieves top-\(k\) evidence from each helper's local index using a BGE-M3 hybrid pipeline, receives sanitized snippets \(\tilde d\), and generates the final answer.
Grounded QA examples are synthesized from retrieval anchors and retained only when evidence spans are found in the source chunk; each record stores its \texttt{chunk\_id} for provenance-based ChunkHit@3.
Query and ground-truth construction details are in \cref{app:rag_construction}; retrieval architecture details are in \cref{app:impl_details}.
We evaluate on single-round sharing as our main setting; the multi-turn deployment extension is discussed in \cref{sec:mas_analysis}.

\paragraph{Training configuration.}
Training runs for 12 rounds with 300 local steps per round per agent, batch size 4, and learning rate \(2\times 10^{-4}\).
We set the proximal regularization strength \(\nu{=}0.1\) to mitigate drift under non-IID data.
Unless otherwise stated, we use \(\lambda_{\text{adv}}{=}1.0\), GRL strength \(\gamma{=}0.5\), \(\lambda_{\text{orth}}{=}0.2\), and prototype noise scale \(\sigma_{\text{noise}}{=}0.01\) (convergence analysis in \cref{app:dp_proto}).
Full hyperparameters and schedules are deferred to \cref{app:doc_types}.

\subsection{Baselines}
\label{sec:baselines}
We compare against practical alternatives for privacy-preserving text sharing:

\textbf{Placeholder-only:} This approach first applies a PII detection model to identify sensitive spans (names, dates, addresses, account numbers, etc.), then replaces each detected span with a type-specific placeholder token (e.g., \texttt{[NAME]}, \texttt{[DATE]}, \texttt{[ADDRESS]}, \texttt{[ACCOUNT]}). We evaluate three PII detectors: \glinermodel~\cite{zaratiana-etal-2024-gliner}, a generalist NER model; \piiranhamodel, a DeBERTa-based model~\cite{he2020deberta} fine-tuned for PII detection; and \debertapiimodel, another DeBERTa variant trained on PII corpora. While placeholder replacement removes explicit identifiers, it does not address implicit stylistic fingerprints; furthermore, opaque placeholder tokens can degrade downstream RAG by collapsing task-relevant spans into generic tokens and weakening answer grounding, provenance, and cross-document aggregation.

\textbf{LLM paraphrasing:} Locally paraphrase text with open-source LLMs using a privacy-focused prompt, then share the rewritten text.

\textbf{Policy gating:} Adapted from dynamic access-control memory sharing~\cite{rezazadeh2025collaborative}, a local policy model (Qwen2.5-7B) decides per chunk whether to share the original text, provide a summary, or refuse sharing, based on the requester's agent tag and the helper's data sensitivity level.

Prompts and policy templates are in \cref{app:baseline_details}.

\subsection{Evaluation Metrics}
\label{sec:metrics}

We evaluate sanitization quality using privacy and utility metrics.

\paragraph{Privacy metrics.}
\textbf{(i) Avg.\ PII}: the average number of PII spans detected in sanitized chunks by an external detector (lower is better).
\textbf{(ii) Ans.\ Rate}: the fraction of final answers that contain at least one exposed PII entity.
We instantiate the external detector as \texttt{gliner-pii-large-v1.0} with a 0.3 confidence threshold over common PII labels; implementation details are in \cref{app:pii_detector_details}.
\textbf{(iii) Distributional fingerprint leakage}: 7-way attribution accuracy/macro-F1 from learned role embeddings (EXP-1) and sanitized text via stylometry (EXP-3).
\textbf{(iv) Prototype fingerprint leakage}: 7-way attribution accuracy/macro-F1 from uploaded prototypes (EXP-2).
Since capability tags and participation are public, these probes do not test whether the helper capability tag is hidden.
They test whether transmitted text or learned artifacts still carry residual source-correlated fingerprints beyond that public information; full protocols are in \cref{sec:attack_eval}.

\paragraph{Utility metrics.}
\textbf{(i) F1/Prec./Rec./Cos.}: token-level F1, precision, recall, and cosine similarity between generated answers and ground-truth answers (bag-of-words TF representation; see \cref{app:impl_details}).
\textbf{(ii) Faithfulness}: the fraction of stopword-removed content words in the answer that appear in retrieved evidence.
\textbf{(iii) ChunkHit@3}: the fraction of ground-truth chunks appearing in top-3 retrieved results.

\subsection{Main Results}
\label{sec:main_results}

\textsc{DiSan} achieves strong privacy protection with modest utility loss. As shown in \cref{tab:main_results}, it reduces answer-level PII exposure from 11.8\% under unprotected sharing to just 0.6\%, while preserving answer faithfulness at 83.17\% close to the 86.10\% unprotected baseline.

\cref{fig:tradeoff} further shows that \textsc{DiSan} offers the most favorable privacy--utility trade-off among the evaluated sanitizers, being the only method that simultaneously achieves sub-1\% answer-level PII leakage and near-baseline task performance.

\begin{figure}[!htb]
\vskip 0.1in
\begin{center}
\centerline{\includegraphics[width=0.9\columnwidth]{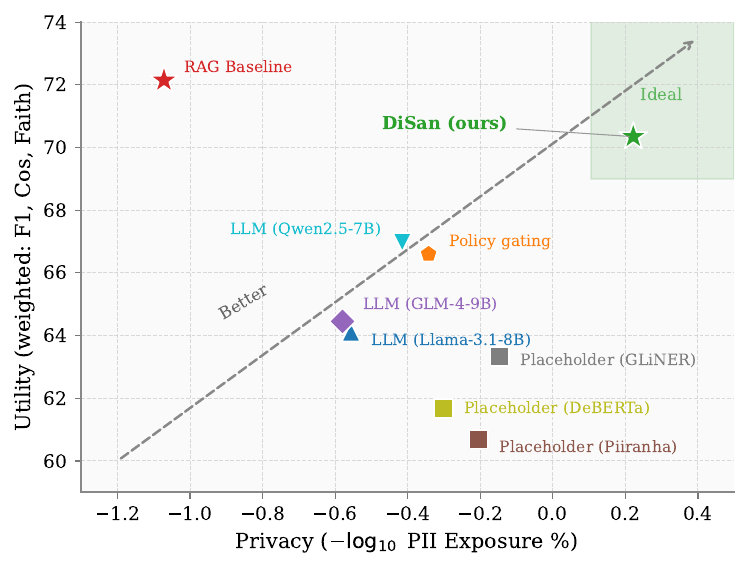}}
\caption{Privacy--utility trade-off. The ideal region (green, upper-right) represents high utility with strong privacy. Among evaluated methods, \textsc{DiSan} achieves the most favorable trade-off on this benchmark.}
\label{fig:tradeoff}
\end{center}
\vskip -0.2in
\end{figure}

\subsection{Ablation Study}
\label{sec:ablation}

We ablate two key components: \emph{style isolation} and \emph{prototype alignment}.
Full results are in \cref{app:ablation_results}.

\begin{table*}[tb]
\caption{\textbf{Ablation results (single-round).}
\textbf{Bold} = best column value; \underline{underline} = second best.
These are not competing sanitizers: A1 and A2 relax privacy constraints, so their higher utility metrics reflect under-sanitization rather than improvement. The relevant comparison is privacy cost vs.\ utility gain.}
\label{tab:ablation_single_round_est1}
\vskip 0.15in
\begin{center}
\begin{small}
\setlength{\tabcolsep}{6pt}
\begin{tabular}{lcccccc}
\toprule
\textbf{Method} &
\textbf{Sem. F1} $\uparrow$ &
\textbf{Cos.} $\uparrow$ &
\textbf{Faith.} $\uparrow$ &
\textbf{ChunkHit@3} $\uparrow$ &
\textbf{Avg. PII} $\downarrow$ &
\textbf{Ans. Rate} $\downarrow$ \\
\midrule
\textbf{\textsc{DiSan} (ours)} & 0.5631 & \underline{0.6558} & 83.17\% & \underline{73.4\%} & \textbf{0.1337} & \textbf{0.6\%} \\
A1: w/o $\mathcal{L}_{\text{orth}}$      & \underline{0.5694} & 0.6309 & \underline{83.46\%} & 72.7\% & 0.9706 & 1.4\% \\
A2: high-$\alpha$ fusion       & \textbf{0.5768} & \textbf{0.6664} & \textbf{84.12\%} & \textbf{74.4\%} & 1.2419 & 2.2\% \\
B: No-ProtoAlign               & 0.5487 & 0.6163 & 81.92\% & 69.8\% & \underline{0.3588} & \underline{0.8\%} \\
\bottomrule
\end{tabular}
\end{small}
\end{center}
\vskip -0.1in
\end{table*}

\paragraph{Style isolation ablations.}
We test: \textbf{(A1)} removing \(\mathcal{L}_{\text{orth}}\) (\(\lambda_{\text{orth}}{=}0\)), allowing role embeddings to absorb stylometric cues; and \textbf{(A2)} high-\(\alpha\) fusion (\(\alpha{=}0.9\)), amplifying style influence during decoding.
A1 and A2 show slightly higher RAG utility metrics because weaker style isolation allows more source-specific lexical content to pass through the decoder, surfacing additional matching tokens in the answer. This is precisely the privacy failure mechanism: the same leaked content that boosts surface utility exposes source-identifying information. A1 increases Avg.\ PII by $7.3\times$ and A2 by $9.3\times$, with answer-level exposure rising to 1.4\% and 2.2\% respectively. The appendix confirms the effect at the representation level: role attribution F1 rises from 0.13 (DiSan) to 0.18 (A1) and 0.24 (A2). These ablations reveal the privacy--utility trade-off: relaxing style isolation shifts the operating point toward higher utility at the cost of substantially weaker privacy. \textsc{DiSan} is selected as the operating point that maximizes privacy while incurring only modest utility loss.

\paragraph{Prototype alignment ablation.}
We remove all prototype components: no EMA prototypes, no global aggregation, and \(\lambda_p{=}0\) (removing \(\mathcal{L}_{\text{proto}}\) from \cref{eq:total_loss}).
Without prototype anchors, role spaces drift under non-IID training, degrading both utility and privacy: cosine similarity drops 6\%, ChunkHit@3 drops 3.6pp, and PII exposure increases 2.7$\times$.

\paragraph{Disentanglement verification.}
Beyond privacy metrics, we verify that the model learns the intended decomposition.
As diagnostic evidence, role embeddings used for sharing are near-random under the 7-way fingerprint probe (F1$=$0.05), while local-only style embeddings remain strongly source-correlated (F1$=$0.84).
\cref{fig:disentanglement} further confirms successful role--style separation: cosine similarity between role and style embeddings clusters near zero across all agents, and a t-SNE projection shows clear geometric separation between the two subspaces.
This indicates that $\mathcal{L}_{\text{orth}}$ concentrates source-correlated variation primarily in the local-only style subspace.
DiSan does not provide formal differential-privacy guarantees; full attack protocols across three surfaces are in \cref{sec:attack_eval}.

\begin{table}[!htb]
\caption{Enron stylometry probe (7 authors). Lower F1 indicates weaker residual distributional fingerprints.}
\label{tab:enron_stylometry_main}
\vskip 0.05in
\begin{center}
\begin{small}
\setlength{\tabcolsep}{3pt}
\begin{tabular}{@{}lcccc@{}}
\toprule
Method & Mask\% & TF-IDF & BERT & Red. \\
\midrule
Raw text       & --     & 0.825 & 0.691 & -- \\
GLiNER masking & 19.2\% & 0.672 & 0.618 & 18.6\% \\
\textsc{DiSan} & --     & \textbf{0.221} & \textbf{0.203} & \textbf{73.2\%} \\
Random         & --     & 0.143 & 0.143 & -- \\
\bottomrule
\end{tabular}
\end{small}
\end{center}
\vskip -0.1in
\end{table}

\cref{tab:enron_stylometry_main} treats Enron as a stylometry diagnostic rather than hidden-tag evaluation; \textsc{DiSan} lowers TF-IDF/BERT attribution to 0.221/0.203, near random and far below GLiNER.
Appendix~\ref{sec:attack_eval} further compares with dedicated authorship-obfuscation baselines, JAMDEC and StyleRemix, which target stylometric leakage rather than RAG utility and yield substantially smaller reductions.

We adopt LoRA~\cite{hu2022lora} for efficient training (details in \cref{app:arch_details,app:training_details}).

\FloatBarrier

\subsection{Multi-Party Collaboration Analysis}
\label{sec:mas_analysis}

\paragraph{Multi-turn RAG.}
The single-round protocol extends naturally to multi-turn settings where the requester refines queries based on accumulated evidence.
At each turn $t$, the requester issues a refined query $q_t$ to a possibly different helper subset; each helper sanitizes new evidence independently before transmitting it.
The key property is preserved: only sanitized text crosses the privacy boundary at every turn.
We treat this as a deployment extension of the same sanitization interface and leave dedicated multi-turn benchmarking to future work.

\paragraph{Case study: cross-organizational IPO analysis.}
\cref{fig:case_study} (see \cref{app:case_study}) illustrates a realistic multi-party scenario where Lumina Capital's Audit-Core agent evaluates a company for IPO eligibility but lacks sufficient external benchmarks.
Audit-Core queries partner agents via a data broker; each external agent sanitizes its response with \textsc{DiSan} before transmission, removing entity names, geographic identifiers, and organizational details while preserving task-relevant financial metrics.
This scenario instantiates the same agent setting used in our benchmark: routing is based on public capability tags, while each helper keeps its repository and retrieval index local.
The privacy boundary is therefore the transmitted sanitized snippet, not the helper identity or the fact of collaboration.
The final answer aggregates benchmark figures without exposing any source-identifying information, enabling accurate cross-party assessment across organizational boundaries.

%% file: sections/conclusion.tex
\section{Conclusion}
Identifier-level anonymization is insufficient for source-invariant text sharing because private organizational information is often a distributional property of text rather than a set of localized identifiers. \textsc{DiSan} addresses this by enforcing role--style orthogonality and federated role alignment, producing sanitized text that preserves task semantics while suppressing source-identifying patterns. Across distributed-agent RAG and Enron stylometry evaluations, the results show that representation-level disentanglement provides a practical path toward safer text sharing across distributed agents. Future work should study stronger adaptive adversaries, repeated-query settings, and broader cross-domain deployments.

%% file: sections/limitations.tex
\section*{Limitations}

\paragraph{Dataset.}
Our main experiments use a synthetic finance corpus with annotated PII spans.
The use of synthetic data is a necessary constraint rather than a methodological choice: datasets containing real PII cannot legally or ethically be used for research publication, and this is standard practice in privacy-preserving NLP.
Among publicly available PII-annotated corpora, most are unsuitable for RAG evaluation due to very short texts (social media, medical notes) or insufficient document-type diversity for multi-party simulation; the selected corpus provides the document length, PII annotation quality, and domain variety required by our setup.
RAG evaluation queries are synthesized on top of this validated base rather than generating both PII data and queries from scratch, which would compound synthesis risk across two stages.
External validation on the Enron email corpus, with real stylistic variation outside the finance domain, partially addresses domain generalizability.

\paragraph{Formal privacy guarantees.}
\textsc{DiSan} does not provide formal differential privacy guarantees.
Applying DP to sequence-to-sequence generation requires per-token noise calibration incompatible with coherent text generation; we instead rely on empirical validation across three attack surfaces.
Formal privacy analysis for generative sanitizers is an important direction for future work.

%% file: sections/appendix_threat_model.tex
\section{Threat Model Details}
\label{app:threat}

This section expands the threat scope summarized in \cref{sec:threat}.

\subsection{Adversary Roles}

We distinguish the primary adversary in the application stage from a narrower diagnostic in the training stage.
In both cases, capability tags and participation are public; privacy is defined as leakage \emph{beyond} these public facts.

\begin{table}[!htb]
\caption{Two-stage threat scope. Capability tags and agent participation are public in both stages.}
\label{tab:threat_scope}
\vskip 0.05in
\begin{center}
\begin{small}
\renewcommand{\arraystretch}{1.12}
\begin{tabularx}{0.98\textwidth}{@{}>{\raggedright\arraybackslash}p{0.15\textwidth}
>{\raggedright\arraybackslash}p{0.18\textwidth}
>{\raggedright\arraybackslash}X
>{\raggedright\arraybackslash}X
>{\raggedright\arraybackslash}p{0.18\textwidth}@{}}
\toprule
Stage & Observer & Visible information & Privacy concern & Evaluation \\
\midrule
Application & Requesting agent & Sanitized snippet, query context, and helper capability tag & Leakage beyond the public tag, including explicit PII, private evidence in the returned text, and organizational fingerprints & PII leakage metrics and text stylometry \\
\addlinespace[2pt]
Training & Federated coordinator or auditor & Model updates and uploaded role prototypes, but not raw text or full local embeddings & Persistent prototype signatures that reveal private corpus structure beyond public participation & Representation and prototype attribution probes \\
\bottomrule
\end{tabularx}
\end{small}
\end{center}
\vskip -0.1in
\end{table}

\paragraph{Primary recipient adversary.}
A collaborating agent receives sanitized text $\tilde{d}$ together with public routing metadata, including the helper's capability tag.
The recipient may try to recover private information from the sanitized content, including explicit PII, stylistic fingerprints tied to the source, organizational document conventions, or other evidence about the helper's private repository beyond the public tag used for routing.
This is the primary adversary addressed by the role and style disentanglement in \textsc{DiSan}.

\paragraph{Illustrative application scenario.}
In the IPO analysis case study (\cref{app:case_study}), Lumina Capital's Audit-Core agent lacks sufficient external evidence and routes subqueries to helper agents using public capability tags such as CorporateBank, MarketForecaster, and AssetManager.
The tag itself is not private, since Audit-Core already knows that a CorporateBank helper is being queried for credit risk evidence.
The privacy risk is that the returned sanitized snippet may reveal information beyond this public tag. Examples include residual company names, account or location identifiers, proprietary report templates, recurring credit assessment language, sector taxonomies, and other organizational fingerprints.
Such leakage could allow the requester to infer private properties of the helper's repository or internal business process even without seeing the raw document.
Our evaluations match these surfaces. PII metrics measure explicit identifier leakage in sanitized outputs and final answers. Stylometric attribution on sanitized text measures distributional fingerprints tied to the source. Representation and prototype attribution probes diagnose whether the learned role space and uploaded prototypes retain signals associated with the source that could support inference beyond public tags.
We do not claim protection against arbitrary reconstruction of a helper's full private corpus.

\paragraph{Training stage prototype observer.}
During federated training, the coordinator observes model updates and uploaded role prototypes but \emph{not} raw text or full local embeddings.
We do not aim to hide which public agent participates in training from the coordinator.
Instead, EXP-2 asks whether uploaded prototypes become persistent source fingerprints that reveal private distributional properties of an agent's repository beyond its public tag.
Because agents hold non-IID corpora, their prototypes can reflect distinctive document type mixtures, role and entity frequencies, and business process patterns. For example, a CorporateBank agent centered on counterparty risk and credit assessments may induce different role space directions than a MarketForecaster agent centered on equity forecasts and price targets.
If such directions remain stable across rounds, an observer can link training artifacts over time or across tasks and infer how an agent's private corpus is organized, without seeing raw text.
\textsc{DiSan} addresses this through adversarial training on prototypes, expressed by $\mathcal{L}_{\text{adv}}$, and Gaussian noise perturbation with $\sigma{=}0.01$ before upload. Together, these mechanisms suppress the directional signatures evaluated in EXP-2 (\cref{sec:attack_eval}).

\paragraph{Out-of-scope adversaries.}
We do not address active adversaries (malicious servers, poisoning, prompt injection) or collusion between multiple recipients and the coordinator.
These represent important directions for future work but are orthogonal to the sanitization objective studied here.

\subsection{Validation Scope}

\Cref{sec:threat} summarizes the validation goals in the main text.
The detailed protocols are reported in \cref{sec:main_results,sec:attack_eval}: PII metrics evaluate explicit identifier leakage, stylometry evaluates source signatures in transmitted text, and embedding/prototype attribution probes diagnose whether learned sharing artifacts carry persistent source information.

%% file: sections/appendix_method_details.tex
\section{Additional Method Details}
\label{app:method_details}

This appendix collects components omitted from the main paper for space, including (i) fusion/residual design,
(ii) expanded prototype objectives and prototype-level adversarial training, and (iii) the federated training procedure.

\subsection{Fusion and Residual Stabilization}
\label{app:fusion}

After fusing the role and style streams via \(\mathbf{H}_{\text{fused}} = g([\mathbf{Z}_r; \mathbf{Z}_s])\), we apply a residual path to stabilize training:
\begin{equation}
\mathbf{H}_{\text{out}} = \alpha \cdot \mathbf{H}_{\text{fused}} + (1-\alpha)\cdot \mathbf{H},
\label{eq:residual}
\end{equation}
where \(\alpha=\sigma(a)\in(0,1)\) is a learnable scalar gate initialized so that \(\alpha\) is small early in training.

For completeness, we provide the gradient decomposition induced by this residual gate.
Let \(\theta_{\text{enc}}\) denote the parameters of the pretrained encoder backbone (the component that produces \(\mathbf{H}\) from input tokens). Note that \(\theta_{\text{enc}}\) \emph{excludes} the fusion layer \(g\), gating parameters \(\{a, \alpha\}\), and projection heads \(\mathbf{W}_r, \mathbf{W}_s\). Under this convention:
\begin{equation}
\frac{\partial \mathcal{L}}{\partial \theta_{\text{enc}}} =
\frac{\partial \mathcal{L}}{\partial \mathbf{H}_{\text{out}}}
\left[
    \alpha \frac{\partial \mathbf{H}_{\text{fused}}}{\partial \mathbf{H}} + (1-\alpha) \mathbf{I}
\right]
\frac{\partial \mathbf{H}}{\partial \theta_{\text{enc}}}.
\label{eq:grad_decomp}
\end{equation}
The residual path helps maintain stable gradients early in training when the role/style streams are still adapting.

\subsection{Expanded Prototype Objectives}
\label{app:proto_details}

\paragraph{Batch role centroid.}
For each type \(k\in\mathcal{K}\), agent \(c\) computes the batch role centroid as a token-level average:
\begin{equation}
\bar{\mathbf{z}}_{r,k}^{(c)} = \frac{1}{|\Omega_k^{(c)}|}\sum_{(i,t)\in \Omega_k^{(c)}} \mathbf{z}^{(i)}_{r,t}, \quad
\Omega_k^{(c)} = \{(i,t) : \ell_{i,t} = k\}.
\label{eq:batch_role_centroid}
\end{equation}
If \(\Omega_k^{(c)}=\emptyset\), we skip type \(k\) in the corresponding loss term.

\paragraph{Decomposed alignment terms.}
The main paper uses the spherical cosine alignment in \cref{eq:proto_loss}.
Equivalently, one can view prototype alignment as matching (i) centroid distance, (ii) angular direction, and (iii) within-batch dispersion, defined per type \(k\in\mathcal{K}\):
\begin{align}
\mathcal{L}_{\text{align}}
&= \sum_{k\in\mathcal{K}} \left\| \bar{\mathbf{z}}_{r,k}^{(c)} - \boldsymbol{\mu}^*_{k} \right\|_2^2, \label{eq:align} \\
\mathcal{L}_{\text{cos}}
&= \sum_{k\in\mathcal{K}} \left(1 - \frac{\bar{\mathbf{z}}_{r,k}^{(c)\top} \boldsymbol{\mu}^*_{k}}{\|\bar{\mathbf{z}}_{r,k}^{(c)}\|_2 \cdot \|\boldsymbol{\mu}^*_{k}\|_2 + \epsilon}\right), \label{eq:cos} \\
\mathcal{L}_{\text{var}}
&= \sum_{k\in\mathcal{K}}\frac{1}{|\Omega_k^{(c)}|}\sum_{(i,t)\in\Omega_k^{(c)}}
\left\| \mathbf{z}_{r,t}^{(i)} - \bar{\mathbf{z}}_{r,k}^{(c)} \right\|_2^2, \label{eq:var} \\
\mathcal{L}_{\text{proto(full)}}
&= \mathcal{L}_{\text{align}} + \mathcal{L}_{\text{cos}} + \lambda_{\text{var}}\,\mathcal{L}_{\text{var}}. \label{eq:proto_full}
\end{align}
Here \(\Omega_k^{(c)}=\{(i,t):\ell_{i,t}=k\}\) is the token index set for type \(k\) in agent \(c\)'s batch (\cref{eq:batch_role_centroid}).
In practice, the spherical objective in \cref{eq:proto_loss} is a compact alternative that avoids redundant hyperparameters.

\paragraph{EMA prototype update.}
Each agent maintains an EMA prototype for each type:
\begin{equation}
\boldsymbol{\mu}^{(c)}_{k,s}
=\beta\,\boldsymbol{\mu}^{(c)}_{k,s-1}+(1-\beta)\,\bar{\mathbf{z}}^{(c)}_{r,k,s},
\label{eq:ema}
\end{equation}
where \(s\) indexes steps and \(\beta\in(0,1)\).

\paragraph{Spherical EMA.}
When using spherical alignment, we maintain the running prototype on the unit sphere:
\begin{equation}
\hat{\boldsymbol{\mu}}^{(c)}_{k,s}
=\mathrm{normalize}\!\left(\beta\,\hat{\boldsymbol{\mu}}^{(c)}_{k,s-1}+(1-\beta)\,\hat{\bar{\mathbf{z}}}^{(c)}_{r,k,s}\right),
\label{eq:spherical_ema}
\end{equation}
where \(\hat{\bar{\mathbf{z}}}^{(c)}_{r,k,s}=\mathrm{normalize}(\bar{\mathbf{z}}^{(c)}_{r,k,s})\).

\paragraph{Prototype adversarial training.}
As a diagnostic in the training stage, uploaded prototypes may carry persistent distributional signatures tied to individual sources beyond public agent tags.
We therefore train a prototype discriminator \(D_\psi\) to predict the source of uploaded prototypes, and train the encoder to \emph{fool} this discriminator using the GRL objective in \cref{eq:adv}.

To enable gradient flow despite EMA, we use an estimate that preserves gradients:
\begin{equation}
\bar{\boldsymbol{\mu}}^{(c)}_{k}
= (1-\eta)\,\mathrm{sg}(\boldsymbol{\mu}^{(c)}_{k})+\eta\,\bar{\mathbf{z}}^{(c)}_{r,k},
\label{eq:grad_proto}
\end{equation}
where \(\mathrm{sg}(\cdot)\) stops gradients and \(\eta\in[0,1]\). Note that \(\bar{\boldsymbol{\mu}}^{(c)}_{k}\) is used \emph{only for backpropagation} and is not uploaded to the server.

The discriminator is trained on the perturbed prototypes that are uploaded:
\begin{equation}
\mathcal{L}_{\text{disc}}
=\sum_{c=1}^{C}\sum_{k\in\mathcal{K}}
\mathrm{CE}\!\left(D_\psi(\tilde{\boldsymbol{\mu}}^{(c)}_{k}),\,c\right),
\label{eq:disc}
\end{equation}
where \(\tilde{\boldsymbol{\mu}}^{(c)}_{k}\) is the perturbed prototype uploaded to the server, as defined in \cref{eq:noise}.

On the client side, we apply GRL to the estimate that preserves gradients and optimize the adversarial loss:
\begin{equation}
\mathcal{L}_{\text{adv}}
=\sum_{k\in\mathcal{K}}
\mathrm{CE}\!\left(D_\psi(\mathrm{GRL}_\gamma(\bar{\boldsymbol{\mu}}^{(c)}_{k})),\,c\right),
\label{eq:proto_adv}
\end{equation}
where \(\gamma\) is the GRL strength used in \cref{eq:adv}. The adversarial loss uses \(\bar{\boldsymbol{\mu}}^{(c)}_{k}\) to ensure gradient flow, while the discriminator \(D_\psi\) is trained on the perturbed uploaded prototypes \(\tilde{\boldsymbol{\mu}}^{(c)}_{k}\) to match the actual attack surface.

\subsection{Model Architecture Details}
\label{app:arch_details}

\paragraph{Backbone.}
\textsc{DiSan} uses LongT5-TGlobal-Base as its backbone.
The encoder combines local sliding-window attention with dynamically constructed global tokens, enabling linear-complexity modeling of sequences up to 16,384 tokens while preserving long-range context.
The backbone hidden size is denoted \(d_{\text{enc}}\); in our implementation the maximum input length is 1,536 tokens.
The decoder is the LongT5 Transformer decoder and generates sanitized text autoregressively from modified encoder states. Thus \textsc{DiSan} does not replace the pretrained decoder with a separate generator; it inserts a role--style bottleneck between the pretrained encoder and decoder.

\paragraph{Projection heads and discriminator.}
LoRA adapters with rank 8 are applied to the attention Q and V projections of the encoder, adding approximately 0.9M trainable parameters.
Given encoder states \(\mathbf{H}\in\mathbb{R}^{T\times d_{\text{enc}}}\), the role and style projection heads are linear layers followed by dropout with \(p=0.1\) that map each token to 256-dimensional subspaces:
$\mathbf{W}_r, \mathbf{W}_s \in \mathbb{R}^{d_{\text{enc}} \times 256}$ (\cref{eq:role_proj,eq:style_proj}).
The two streams are concatenated and mapped back to \(d_{\text{enc}}\) by a linear fusion layer \(g:\mathbb{R}^{512}\to\mathbb{R}^{d_{\text{enc}}}\).
We use the projected-residual variant by default: the original encoder state is first compressed through the same 512-dimensional bottleneck and projected back to \(d_{\text{enc}}\), then mixed with the fused representation using a learnable scalar initialized to 0.5. This preserves generation stability without passing raw encoder states directly to the decoder.

\paragraph{Adversarial classifiers.}
All source attribution discriminators use the same MLP template, input \(\to 128 \to 128 \to C\), with ReLU activations and dropout 0.1 after each hidden layer; \(C=7\) in our experiments.
The role discriminator takes mean-pooled role embeddings with 256 dimensions, and the fused discriminator takes mean-pooled fused encoder states with dimension \(d_{\text{enc}}\). Both are trained through a gradient reversal layer to discourage signals tied to individual sources from being encoded in the shareable role stream or in the final decoder input.
The prototype discriminator \(D_\psi\) uses the same MLP architecture with a 256-dimensional input and is trained on uploaded role prototypes; the client-side adversarial loss applies GRL to a prototype estimate with gradients so that local updates make prototypes less predictive of their source agent.

\subsection{Federated Training Details}
\label{app:training_details}

\paragraph{Parameter-efficient fine-tuning.}
We employ LoRA~\cite{hu2022lora} for communication-efficient federated training.
Specifically, we freeze the pretrained encoder backbone and apply low-rank adapters only to the attention query and value projections (target modules: \texttt{q}, \texttt{v}).
The trainable parameters include:
(i) LoRA adapters ($\approx$0.9M parameters);
(ii) role projection $\mathbf{W}_r$ and style projection $\mathbf{W}_s$;
(iii) fusion layer $g(\cdot)$ and residual components;
(iv) adversarial classifiers.
Note that when using LoRA, the gradient decomposition in \cref{eq:grad_decomp} flows only through the LoRA-adapted attention layers (not the frozen FFN and normalization layers), while the projection heads $\mathbf{W}_r$, $\mathbf{W}_s$ receive full gradients.

\paragraph{Communication cost.}
LoRA substantially reduces communication cost compared to full-model synchronization.
The pretrained encoder backbone is frozen; only LoRA adapters (applied to attention Q/V projections) and the disentanglement heads are trainable.
During each round, agents upload only trainable parameters: LoRA adapters ($\approx$0.9M parameters, $\approx$5\,MB), role projection head $\mathbf{W}_r$, and fusion layer $g$. Style projection weights $\mathbf{W}_s$ remain strictly local and are never transmitted. In addition, each agent uploads and downloads a set of role prototypes of size \(|\mathcal{K}|\times d_r\) per round (256-dimensional centroids for $|\mathcal{K}|$ entity types), which is negligible ($\approx$1\,KB) compared to model synchronization.

\paragraph{Convergence analysis.}
\label{app:theory_convergence}
Our method builds on FedProx~\cite{li2020federated}, which incorporates a proximal term $\frac{\nu}{2} \|\theta - \theta^*\|^2$ to control client drift under non-IID data. Under standard assumptions (L-smoothness, bounded gradient variance), FedProx guarantees convergence to a stationary point~\cite{li2020federated}. Our additional loss components ($\mathcal{L}_{\text{orth}}$, $\mathcal{L}_{\text{proto}}$) are smooth regularizers that preserve this guarantee. The adversarial component $\mathcal{L}_{\text{adv}}$ introduces additional complexity; we rely on empirical validation of stable convergence through training curves and loss monitoring across all federated rounds.

\paragraph{FedProx.}
We add a proximal regularizer during updates:
\begin{equation}
\mathcal{L}_{\text{prox}}=\frac{\nu}{2}\left\|\theta-\theta^*\right\|_2^2.
\label{eq:prox}
\end{equation}

\paragraph{Prototype perturbation.}
Before uploading, agents may apply \(\ell_2\)-normalization and add Gaussian noise:
\begin{equation}
\tilde{\boldsymbol{\mu}}^{(c)}_{k}
=\mathrm{normalize}\!\left(\boldsymbol{\mu}^{(c)}_{k}+\mathcal{N}(\mathbf{0},\sigma_{\text{noise}}^2\mathbf{I})\right).
\label{eq:noise}
\end{equation}

\paragraph{Pseudocode.}
Algorithm~\ref{alg:training} summarizes the end-to-end federated training procedure.

\begin{algorithm}[tb]
\caption{\textsc{DiSan} Training}
\label{alg:training}
\begin{algorithmic}
\STATE {\bfseries Input:} Agents \(\{1,\ldots,C\}\) with local data \(\{\mathcal{D}_c\}\), rounds \(R\), local steps \(K\)
\STATE Initialize global model \(\theta^*\) and prototype discriminator \(D_\psi\)
\STATE Initialize global prototypes \(\{\boldsymbol{\mu}^*_{k}\}_{k\in\mathcal{K}}\) (e.g., zeros)
\FOR{\(\rho = 1\) {\bfseries to} \(R\)}
    \STATE Broadcast \(\theta^*\), \(D_\psi\), and \(\{\boldsymbol{\mu}^*_{k}\}_{k\in\mathcal{K}}\) to all agents
    \FOR{agent \(c = 1\) {\bfseries to} \(C\) \textbf{in parallel}}
        \STATE \(\theta^{(c)} \gets \theta^*\)
        \FOR{\(s = 1\) {\bfseries to} \(K\)}
            \STATE Sample a batch from \(\mathcal{D}_c\)
            \STATE Compute \(\mathcal{L}^{(c)}_{\text{local}}\) (\cref{eq:total_loss})
            \STATE Update \(\theta^{(c)}\) via gradient descent
            \STATE Update running prototypes \(\{\boldsymbol{\mu}^{(c)}_{k}\}_{k\in\mathcal{K}}\) (\cref{eq:ema} or \cref{eq:spherical_ema})
        \ENDFOR
        \STATE Optional perturbation: \(\tilde{\boldsymbol{\mu}}^{(c)}_{k} \gets \mathrm{perturb}(\boldsymbol{\mu}^{(c)}_{k})\ \forall k\in\mathcal{K}\)
        \STATE Upload \(\big(\theta^{(c)}, \{\tilde{\boldsymbol{\mu}}^{(c)}_{k}\}_{k\in\mathcal{K}}, \{N_{c,k}\}_{k\in\mathcal{K}}\big)\) to server
    \ENDFOR
    \STATE \(\theta^* \gets \texttt{weighted\_avg}(\{\theta^{(c)}\}_{c=1}^{C})\)
    \FOR{\(k\in\mathcal{K}\)}
        \STATE \(\boldsymbol{\mu}^*_{k} \gets
        \dfrac{\sum\limits_{c=1}^{C} N_{c,k}\,\tilde{\boldsymbol{\mu}}^{(c)}_{k}}
              {\sum\limits_{c=1}^{C} N_{c,k}}\)
    \ENDFOR
    \STATE Train \(D_\psi\) on \(\{(\tilde{\boldsymbol{\mu}}^{(c)}_{k}, c)\}_{c,k}\) using \cref{eq:disc}
\ENDFOR
\STATE {\bfseries Return} \(\theta^*\)
\end{algorithmic}
\end{algorithm}

%% file: sections/appendix_experimental_supplement.tex
\section{Experimental Supplement}
\label{app:exp_details}

\subsection{Attack Evaluation}
\label{sec:attack_eval}

We empirically evaluate privacy on three surfaces: sanitized output text, learned representations, and uploaded prototypes, as shown in \cref{tab:attack_embed_proto,tab:attack_stylometry}.
The primary surface in the application stage is sanitized text received by another agent.
Capability tags and participation are public in our threat model, so the 7-way probes are not meant to hide the helper capability tag itself.
They serve as diagnostics for whether sanitized text, role representations, or uploaded prototypes still carry residual source-correlated distributional fingerprints beyond that public information.
Style representations remain strictly local and are reported only as a sanity check that variation identifying agents has been isolated into the local stream.

\paragraph{EXP-1: Embedding Attribution.}
We employ an SVM classifier with RBF kernel to probe whether embeddings retain source-correlated fingerprints beyond public tags.
For \textbf{role embeddings}, which form the internal stream from which sanitized text and prototypes are derived, the SVM achieves close to random performance, with F1$=$0.05 and Acc$=$0.18. This indicates successful privacy protection.
As a disentanglement sanity check, \textbf{style embeddings} remain strictly local and are never transmitted. They achieve high accuracy, with Acc$=$0.89 and F1$=$0.84, confirming that $\mathcal{L}_{\text{orth}}$ isolates source-correlated variation into the local style stream.
This contrast validates the design. Style carries agent fingerprints but never crosses the privacy boundary, while only sanitized text $\tilde{d}$ derived from role representations is transmitted.
Matched distribution tests confirm robustness: role F1 remains at 0.05 while style F1 stays at 0.80.

\paragraph{EXP-2: Prototype Attribution.}
We test whether uploaded role prototypes expose stable fingerprints tied to individual sources beyond public agent tags through two probes.
\textbf{Sample-to-Proto} predicts which client prototype is closest to a training sample embedding.
\textbf{Bootstrap-Proto} trains a probe on local client data and tests it on uploaded prototypes.
Both achieve close to random performance, with Acc$=$0.14--0.16 and F1$\approx$0.09. The 95\% confidence intervals include zero, indicating that prototype defenses, including normalization, noise perturbation, and adversarial training, suppress persistent prototype signatures.
Cross-round linkage attacks train on round-1 prototypes and test on round-12 prototypes. They achieve F1$=$0.10, confirming that temporal linkage is also ineffective.

\paragraph{EXP-3: Text Stylometry.}
We evaluate stylometric leakage in sanitized text using two complementary probes: (i) TF-IDF features (5000 dimensions, unigrams + bigrams) with classical classifiers (Logistic Regression, LinearSVC, Random Forest); and (ii) a neural encoder probe (\texttt{all-MiniLM-L6-v2})~\cite{reimers2019sentencebert,wang2020minilm} capturing deeper contextual patterns beyond surface n-grams.
These probes use source labels as a diagnostic signal for residual distributional fingerprints, rather than as a claim that public capability tags are hidden.

\textbf{EXP-3a (Synthetic Finance):} The full 7-way test shows high F1 ($\approx$0.90) for both raw and sanitized text, reflecting document-type confounding rather than stylistic signals.
Controlled binary tests isolating stylometric variation show attribution F1 of 0.51--0.56 (sanitized) vs.\ 0.54--0.57 (raw), both near the 0.50 random baseline, suggesting limited intrinsic style variation in the synthetic data.

\textbf{EXP-3b (Enron Emails, External Validation):} To validate on real-world data with genuine author variation, we evaluate on the Enron email corpus~\cite{klimt2004enron} (7 authors, 500 emails each; details in \cref{app:enron_details}).
On raw emails, TF-IDF achieves 82.5\% F1 and the BERT probe achieves 69.1\% F1 (vs.\ 14.3\% random baseline).

We also run stylometric attribution after placeholder masking with each PII detector from our baselines, to directly test whether token-level removal resolves distributional leakage.
Even GLiNER, the strongest detector (masking 19.2\% of tokens), reduces TF-IDF attribution to 67.2\% (18.6\% reduction) and BERT attribution to 61.8\% (10.6\% reduction).
After \textsc{DiSan} sanitization, TF-IDF drops to 22.1\% (\textbf{73.2\% reduction}) and the BERT probe drops to 20.3\% (\textbf{70.6\% reduction}), both approaching the random baseline.
The gap confirms that distributional signatures persist even after aggressive masking, and that a representation-level approach is necessary to suppress them.

\textbf{EXP-3c (Authorship Obfuscation Baselines):}
We compare against two recent authorship obfuscation methods on the identical Enron setup: JAMDEC~\cite{fisher-etal-2024-jamdec}, which applies constrained decoding over a small language model to suppress author-specific tokens, and StyleRemix~\cite{fisher-etal-2024-styleremix}, which perturbs fine-grained style elements via LoRA modules in two modes (Fixed: uniform style target; Adaptive: per-author target).
Results are in \cref{tab:attack_obfuscation}.

JAMDEC barely moves the attribution needle (TF-IDF F1: $0.825\to0.816$, 1.2\% reduction), suggesting that constrained decoding without explicit disentanglement leaves distributional fingerprints intact.
StyleRemix Fixed achieves 11.5\% TF-IDF reduction but only 3.4\% under the neural BERT probe, indicating that surface n-gram perturbation does not fully remove contextual attribution cues.
StyleRemix Adaptive \emph{backfires}: by steering each author toward a distinct style target it introduces new per-author signatures, increasing TF-IDF F1 to 0.970 and BERT F1 to 0.845, both above the raw baseline.
\textsc{DiSan} delivers 73.2\% TF-IDF reduction and 70.6\% BERT reduction, approximately $60\times$ greater than JAMDEC and $6\times$ greater than the best StyleRemix variant, by enforcing explicit role--style disentanglement rather than surface-level rewriting.

\begin{table}[!htb]
\caption{\textbf{EXP-3c: Authorship obfuscation baselines on Enron Emails (7-way).}
Reduction is relative to the 0.825 TF-IDF / 0.691 BERT raw baseline. Negative values indicate worse-than-raw attribution (backfire). Random baseline: F1$=$0.143.}
\label{tab:attack_obfuscation}
\vskip 0.05in
\begin{center}
\begin{small}
\setlength{\tabcolsep}{4pt}
\renewcommand{\arraystretch}{1.05}
\begin{tabular}{@{}lcccc@{}}
\toprule
Method & TF-IDF F1 & TF-IDF Red. & BERT F1 & BERT Red. \\
\midrule
Raw text              & 0.825 & --       & 0.691 & --       \\
JAMDEC                & 0.816 & 1.2\%    & 0.696 & $-$0.6\% \\
StyleRemix Fixed      & 0.737 & 11.5\%   & 0.664 & 3.4\%    \\
StyleRemix Adaptive   & 0.970 & $-$20.8\% & 0.845 & $-$27.9\% \\
\textsc{DiSan} (ours) & \textbf{0.221} & \textbf{73.2\%} & \textbf{0.203} & \textbf{70.6\%} \\
Random baseline       & 0.143 & --       & 0.143 & --       \\
\bottomrule
\end{tabular}
\end{small}
\end{center}
\vskip -0.1in
\end{table}

\paragraph{Attack model scope.}
Our evaluation employs a principled hierarchy: (i) classical stylometry (TF-IDF + linear classifiers)~\cite{stamatatos2009survey} for reproducibility; (ii) neural encoder probes (\texttt{all-MiniLM-L6-v2}) capturing deeper contextual patterns; and (iii) dedicated authorship obfuscation baselines (JAMDEC, StyleRemix) that directly target distributional fingerprints.
Consistency across all levels (F1$\approx$0.05 for embeddings, 73.2\% stylometric reduction on Enron, $6\text{--}60\times$ greater reduction than obfuscation baselines) provides converging evidence of effective privacy protection under the honest-but-curious threat model.
Stronger attacks (LLM-based attribution, adaptive adversaries with repeated queries) remain important future directions but exceed typical honest-but-curious capabilities.

\begin{table*}[!htb]
\caption{\textbf{Attack evaluation: Embedding and Prototype Attribution (EXP-1, EXP-2).} Random baseline: Acc$=$F1$=$0.14.}
\label{tab:attack_embed_proto}
\vskip 0.1in
\begin{center}
\begin{small}
\begin{minipage}[t]{0.46\textwidth}
\centering
\setlength{\tabcolsep}{5pt}
\begin{tabular}{@{}lcccc@{}}
\toprule
\multicolumn{5}{c}{\textbf{EXP-1: Embedding Attribution (SVM, 7-way)}} \\
\midrule
Embedding & Acc & F1 & Match F1 & Note \\
\midrule
\texttt{role} (shared)  & 0.18 & \textbf{0.05} & 0.05 & Near-random \\
\texttt{style} (local)  & 0.89 & 0.84 & 0.80 & Validated \\
\bottomrule
\end{tabular}
\end{minipage}%
\hfill
\begin{minipage}[t]{0.46\textwidth}
\centering
\setlength{\tabcolsep}{4pt}
\begin{tabular}{@{}lcccc@{}}
\toprule
\multicolumn{5}{c}{\textbf{EXP-2: Prototype Attribution}} \\
\midrule
Attack & Acc & F1 & Boot Acc & 95\% CI \\
\midrule
Sample-to-Proto     & 0.16 & 0.09 & --            & -- \\
Bootstrap-Proto     & 0.14 & 0.07 & 0.13$\pm$0.13 & [0.00, 0.43] \\
\bottomrule
\end{tabular}
\end{minipage}
\end{small}
\end{center}
\vskip -0.15in
\end{table*}

\begin{table*}[!htb]
\caption{\textbf{Attack evaluation: Text Stylometry (EXP-3a, EXP-3b).} TF-IDF and BERT probes on synthetic finance and Enron emails. EXP-3b also includes placeholder-masking baselines to directly test whether token-level removal resolves distributional leakage. Random baseline: F1$=$0.143.}
\label{tab:attack_stylometry}
\vskip 0.1in
\begin{center}
\begin{small}
\begin{minipage}[t]{0.38\textwidth}
\centering
\textbf{EXP-3a: Synthetic Finance (7-way)}
\vskip 0.05in
\setlength{\tabcolsep}{5pt}
\begin{tabular}{@{}lccc@{}}
\toprule
Setting & Best F1 & Random & Controlled \\
\midrule
Raw (TF-IDF)            & 0.90 & 0.14 & 0.54--0.57 \\
Raw (BERT)              & 0.87 & 0.14 & 0.52--0.55 \\
\textsc{DiSan} (TF-IDF) & 0.89 & 0.14 & 0.51--0.56 \\
\textsc{DiSan} (BERT)   & 0.86 & 0.14 & 0.49--0.53 \\
\bottomrule
\end{tabular}
\end{minipage}%
\hfill
\begin{minipage}[t]{0.58\textwidth}
\centering
\textbf{EXP-3b: Enron Emails (7-way)}
\vskip 0.05in
\setlength{\tabcolsep}{4pt}
\begin{tabular}{@{}lcccc@{}}
\toprule
Method & Mask\% & TF-IDF F1 & BERT F1 & TF-IDF Red. \\
\midrule
Raw text                & --    & 0.825 & 0.691 & --              \\
+Piiranha               & 5.3\% & 0.785 & 0.669 & 4.9\%           \\
+DeBERTa                & 17.5\% & 0.744 & 0.649 & 9.9\%          \\
+GLiNER                 & 19.2\% & 0.672 & 0.618 & 18.6\%         \\
\textsc{DiSan}          & --    & \textbf{0.221} & \textbf{0.203} & \textbf{73.2\%} \\
Random                  & --    & 0.143 & 0.143 & --              \\
\bottomrule
\end{tabular}
\end{minipage}
\end{small}
\end{center}
\vskip -0.15in
\end{table*}

\subsection{Attack Evaluation Details}
\label{sec:attack_details}
For all embedding attribution attacks (EXP-1), we use Support Vector Machines (SVM) with RBF kernel.

For prototype attribution (EXP-2), we evaluate two attacks:
\begin{itemize}
    \item \textbf{Sample-to-Proto}: For each test sample, compute its role embedding centroid and measure cosine similarity to each client's global prototype; predict the client with highest similarity.
    \item \textbf{Bootstrap-Proto (MLP)}: Train an MLP classifier on prototype embeddings from training rounds and evaluate on held-out rounds, testing cross-round linkability.
\end{itemize}

\subsection{Enron Email Experiment Details}
\label{app:enron_details}

To validate stylometric protection on real-world data with genuine author variation, we conduct an external evaluation on the Enron email corpus~\cite{klimt2004enron}. This dataset contains approximately 500,000 emails from 150 Enron employees, released during the 2001 federal investigation. It is widely used as a benchmark for authorship attribution and email classification research.

\paragraph{Data selection and preprocessing.}
We select the top 7 senders by email volume to match the number of agents in our main experiments:
\texttt{kaminski-v} (20,123 emails), \texttt{mann-k} (16,891), \texttt{dasovich-j} (16,359), \texttt{jones-t} (15,491), \texttt{kean-s} (15,352), \texttt{shackleton-s} (14,076), and \texttt{farmer-d} (9,869).
For each sender, we randomly sample 500 emails (stratified), yielding 3,500 total samples.

Preprocessing steps:
\begin{itemize}
    \item Extract email body by removing headers (To, From, Subject, Date, etc.)
    \item Remove forwarded message markers and quoted reply sections
    \item Filter emails by body length: minimum 100 characters, maximum 2,000 characters
\end{itemize}

\paragraph{Feature extraction.}
We use TF-IDF vectorization with the following parameters:
\begin{itemize}
    \item Maximum features: 5,000
    \item N-gram range: unigrams and bigrams
    \item Minimum document frequency: 2
    \item Maximum document frequency: 0.95
\end{itemize}

\paragraph{Classification.}
We evaluate three classifiers commonly used in stylometry research:
\begin{itemize}
    \item \textbf{Logistic Regression}: L2 regularization, max iterations = 1,000
    \item \textbf{LinearSVC}: Linear kernel SVM, max iterations = 1,000
    \item \textbf{Random Forest}: 100 estimators
\end{itemize}

For stronger attack evaluation, we also employ a pre-trained transformer encoder (\texttt{all-MiniLM-L6-v2}) to generate 384-dimensional sentence embeddings, followed by SVM-RBF classification. This neural probe captures deeper contextual patterns beyond surface-level n-grams.

Data is split into 70\% training and 30\% test sets using stratified sampling (random seed = 42). We report the best F1 (macro) across all classifiers.

\paragraph{Sanitization.}
Raw emails are processed through the trained \textsc{DiSan} model using greedy decoding (beam size = 1) with the task prefix \texttt{``deidentify:''}. Both TF-IDF and BERT pipelines are then applied to the sanitized outputs to evaluate residual stylometric leakage.

\paragraph{Results interpretation.}
The 82.5\% F1 on raw emails versus 14.3\% random baseline confirms that Enron emails exhibit strong, distinguishable stylistic patterns across senders, unlike our synthetic finance dataset where controlled tests showed near-random attribution even on raw text.

Placeholder masking offers diminishing returns: even GLiNER, the strongest detector at 19.2\% token removal, reduces TF-IDF F1 to only 67.2\% (18.6\% reduction). This directly demonstrates that source-identifying signals are distributed across the text rather than concentrated in explicit identifiers.

\textsc{DiSan} reduces TF-IDF F1 to 22.1\% (\textbf{73.2\% reduction}, $0.825\to0.221$), nearly 4$\times$ the best masking baseline, demonstrating effective removal of author-identifying patterns when they exist in the source data.

The BERT probe (SVM-RBF on transformer embeddings) yields lower raw F1 (69.1\%) than TF-IDF, suggesting that Enron stylometry relies more on surface-level n-gram patterns than deep semantic structure. Nevertheless, \textsc{DiSan} still achieves 70.6\% reduction ($0.691\to0.203$), showing that the reduction also holds under a neural attribution probe.

\subsection{Agent construction and non-IID statistics}
\label{app:client_split}
We construct agents by doc-type as described in \cref{sec:experiments}.
This induces (i) quantity skew (agents have different numbers of documents), (ii) label skew (doc-type inventories differ by design), and (iii) feature skew (document length and entity-type distributions differ across agents).
In our threat model (\cref{sec:threat}), capability tags (agent identities such as ``AssetManager'') are treated as public; therefore, our privacy evaluation focuses on leakage \emph{beyond} these public priors.

\cref{fig:non_iid} visualizes these heterogeneity patterns across the seven agents.

\begin{figure*}[tb]
\vskip 0.2in
\begin{center}
\centerline{\includegraphics[width=0.85\textwidth]{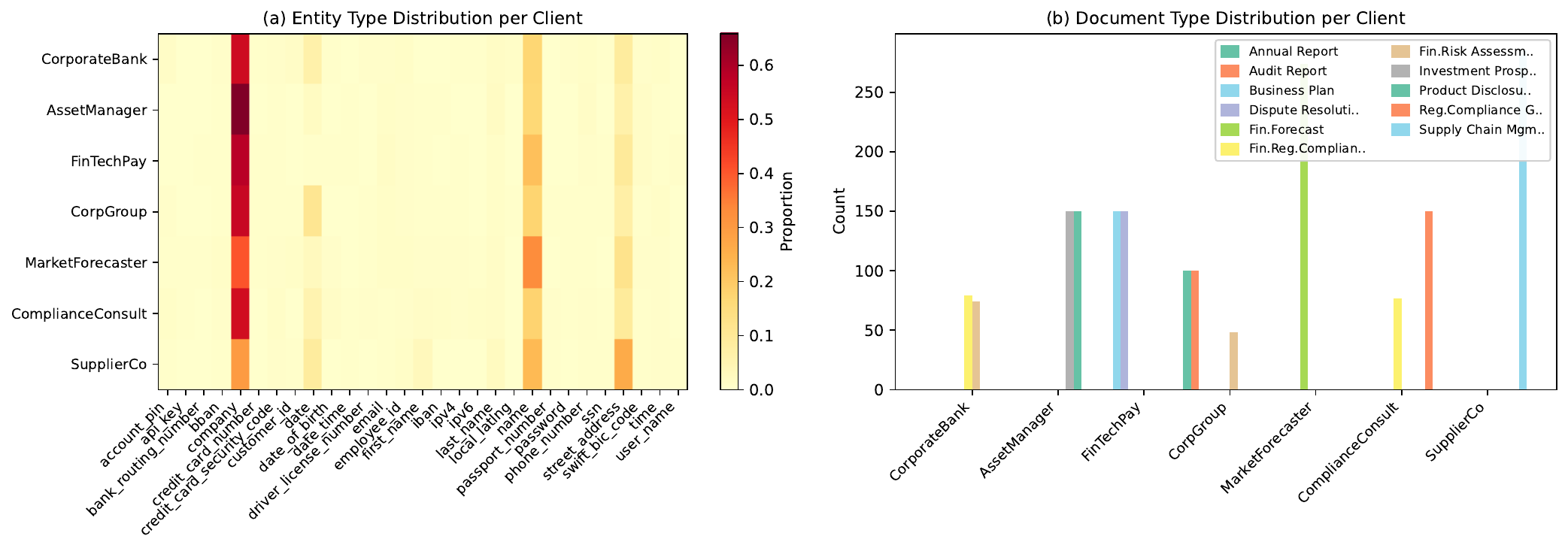}}
\caption{Non-IID data heterogeneity across agents.
(a) Entity type distribution (normalized) shows different agents emphasize different PII categories.
(b) Document type distribution confirms disjoint doc-type inventories by design.}
\label{fig:non_iid}
\end{center}
\vskip -0.2in
\end{figure*}

\subsection{Agent doc-type inventories and hyperparameters}
\label{app:doc_types}

\cref{tab:doc_types_hyper} lists the exact doc-type identifiers used to construct each agent and the complete hyperparameter set for reproducibility.

\begin{table*}[tb]
\caption{(Left) Agent doc-type inventories for \(C{=}7\) agents. (Right) Hyperparameters and schedules.}
\label{tab:doc_types_hyper}
\vskip 0.15in
\begin{center}
\begin{small}
\begin{minipage}[t]{0.48\textwidth}
\centering
\setlength{\tabcolsep}{3pt}
\renewcommand{\arraystretch}{1.05}
\begin{tabular}{@{}
  >{\raggedright\arraybackslash}p{0.32\textwidth}
  >{\raggedright\arraybackslash}p{0.62\textwidth}
@{}}
\toprule
\textbf{Agent} & \textbf{Doc-type identifiers} \\
\midrule
CorporateBank &
\texttt{Financial\_Regulatory\_Compliance \_Report};
\texttt{Financial\_Risk\_Assessment}. \\
AssetManager &
\texttt{Investment\_Prospectus};
\texttt{Product\_Disclosure\_Statement}. \\
FinTechPay &
\texttt{Business\_Plan};
\texttt{Dispute\_Resolution\_Policy}. \\
CorpGroup &
\texttt{Annual\_Report};
\texttt{Audit\_Report};
\texttt{Financial\_Risk\_Assessment}. \\
MarketForecaster &
\texttt{Financial\_Forecast}. \\
ComplianceConsult &
\texttt{Financial\_Regulatory\_Compliance \_Report};
\texttt{Regulatory\_Compliance\_Guide}. \\
SupplierCo &
\texttt{Supply\_Chain\_Management\_Agreement}. \\
\bottomrule
\end{tabular}
\end{minipage}%
\hfill
\begin{minipage}[t]{0.48\textwidth}
\centering
\setlength{\tabcolsep}{4pt}
\renewcommand{\arraystretch}{1.05}
\begin{tabular}{@{}ll@{}}
\toprule
\multicolumn{2}{@{}l}{\textbf{Federated training}} \\
\midrule
Rounds & 12 \\
Local steps per round & 300 \\
Batch size & 4 \\
Learning rate & \(2\times 10^{-4}\) \\
FedProx \(\nu\) & 0.1 \\
\midrule
\multicolumn{2}{@{}l}{\textbf{Loss weights and schedules}} \\
\midrule
\(\lambda_{\text{adv}}\) & 1.0 \\
GRL strength \(\gamma\) & 0.5 \\
\(\lambda_{\text{orth}}\) & 0.2 \\
\(\lambda_{\text{p}}\) (prototype alignment) & 1.0 \\
Proto alignment warmup & 30 steps (round 2) \\
Prototype discriminator steps/round & 200 \\
Prototype noise scale \(\sigma\) & 0.01 \\
\bottomrule
\end{tabular}
\end{minipage}
\end{small}
\end{center}
\vskip -0.1in
\end{table*}

\subsection{Implementation and Evaluation Details}
\label{app:impl_details}

\paragraph{Model architecture.}
\textsc{DiSan} uses LongT5-TGlobal-Base as its backbone; full architecture details are in \cref{app:arch_details}.

\paragraph{RAG evaluation: query and ground-truth construction.}
\label{app:rag_construction}
We construct the RAG benchmark from the sanitized/re-written document records rather than generating free-form queries from scratch.
Each input JSONL record contains a document identifier, domain, document type, PII annotations, and a \texttt{rewritten\_text} field used as the retrieval corpus.
The construction pipeline is:
\begin{enumerate}
    \item \textbf{Chunking.} Documents are split into sentence-aware chunks with a 256-token target length and 50-token overlap. Each chunk keeps its \texttt{uid}, document type, source file, sample index, chunk index, and a stable \texttt{chunk\_id}. Chunks are grouped by document type for generation and by agent identifier for retrieval.
    \item \textbf{Anchor extraction.} For every chunk, a schema-guided LLM prompt extracts retrieval-oriented anchors: role hooks, topic/procedure hooks, deadlines, required items, logic gates, regulations, temporal buckets, and a short summary of the local business rule. The prompt requires anchors to be supported by verbatim or near-verbatim evidence from the chunk and discourages PII-bearing names or addresses unless they are essential to the rule.
    \item \textbf{Grounded QA synthesis.} We sample anchor-annotated chunks with a fixed random seed and ask the LLM to generate one focused question, a concise ground-truth answer, and evidence snippets for each sampled chunk. The question must be answerable using only that chunk; the answer records decision factors such as role, requirement/deadline, and logic gate when they are explicitly present.
    \item \textbf{Validation and provenance.} Generated examples are discarded if the evidence is not found in the source chunk or if the query/answer introduces known unsupported concepts. Each retained record stores the originating \texttt{chunk\_id}, merged anchors, evidence spans, and document metadata, so ChunkHit@3 can be computed against the true retrieval target.
\end{enumerate}
This procedure yields a JSONL file of grounded QA records and a separate set of per-agent context files used by the retrieval services. The ground-truth answers are therefore anchored to verbatim source spans, while the evaluation query is natural language and may require the retriever to recover the correct chunk from an agent's local index.

\paragraph{Retrieval architecture.}
Each agent maintains a private local index built from its own context file; raw documents are not centralized for retrieval.
At evaluation time, the requester first selects candidate helper agents using the public capability tags described in \cref{sec:routing}. Each selected helper executes the same local retrieval stack over its own chunks and then applies the evaluated sharing policy (raw sharing, placeholder masking, paraphrasing, policy gating, or \textsc{DiSan}) before transmitting evidence to the requester.
The local retrieval stack uses a three-stage hybrid pipeline based on BGE-M3.
Stage 1 computes candidate sets from dense semantic embeddings, learned sparse lexical vectors, and ColBERT-style late-interaction vectors.
Stage 2 normalizes and fuses the candidate scores, with adaptive weights that increase the sparse component for keyword-rich queries.
Stage 3 re-ranks the fused candidates with \texttt{bge-reranker-v2-m3}; the final top-\(k\) chunks can be expanded with neighboring chunks from the same document to provide broader context for answer generation.
The requester deduplicates returned chunks by \texttt{chunk\_id}, keeps the highest-scoring evidence up to the evaluation budget, and prompts the answer model to respond only from the received sanitized context.

\paragraph{Artifact licenses and terms.}
We use publicly available datasets, models, and evaluation tools under their respective licenses or terms of use, including the synthetic finance corpus, Enron email corpus, GLiNER, BGE-M3, LongT5, and open-source LLM baselines.
We do not redistribute restricted raw data; released code and derived artifacts are intended for research use.

\paragraph{PII detector for leakage metrics.}
\label{app:pii_detector_details}
For Avg.\ PII and Ans.\ Rate in \cref{tab:main_results}, we evaluate the text exposed to the requester using \texttt{gliner-pii-large-v1.0}~\cite{zaratiana-etal-2024-gliner}, a generalist named-entity detector configured with a PII label set covering names, first/last names, email addresses, phone numbers, street/location addresses, city/state/zip, credit-card and bank-account numbers, SSNs, dates of birth, dates, company/organization names, usernames, IP addresses, URLs, passport numbers, and driver-license numbers.
We use a confidence threshold of 0.3 and count non-overlapping detected spans.
Avg.\ PII is the mean number of detected PII spans per shared/retrieved chunk after the method under evaluation has been applied.
Ans.\ Rate is the percentage of generated final answers for which the same detector finds at least one PII span.
The detector is used only for evaluation; training still relies on the dataset's annotated spans and token-level labels, with optional detector-derived masks used only as auxiliary signals when available.

\paragraph{Utility metrics: cosine similarity.}
The cosine similarity reported in \cref{tab:main_results} uses a bag-of-words term-frequency representation rather than neural embeddings, serving as a lightweight lexical similarity measure that complements token-level F1.
This is distinct from the 768-dimensional BGE-M3 dense embeddings used in retrieval.

\subsection{Baselines: prompts and policies}
\label{app:baseline_details}
For LLM paraphrasing, we apply a privacy-focused prompt that (i) removes or generalizes PII, (ii) preserves relational semantics needed for grounding, and (iii) avoids source-identifying formatting.
For policy gating, following dynamic access-control memory sharing~\cite{rezazadeh2025collaborative}, the policy model decides per chunk based on the requester's agent tag and the helper's data sensitivity level, outputting one of three actions: \emph{share} (return original text), \emph{share summary} (return a LLM-generated summary), or \emph{refuse} (return nothing).
\paragraph{LLM paraphrasing prompt.}
The following prompt is used solely for the RAG utility evaluation reported in \cref{tab:main_results} and discussed in \cref{sec:main_results}, where LLM paraphrasing serves as a PII-removal baseline:

\begin{quote}
\small
\textit{You are an assistant that rewrites English financial or compliance-related text to remove personally identifiable information (PII) while preserving all task-relevant content.}

\textit{Your task is to produce a clear, natural-sounding rewritten version that protects individual privacy while retaining the informational value of the text. The rewritten text should preserve financial metrics, business relationships, temporal context, and domain-specific details that are important for downstream tasks.}

\textit{Guidelines:}
\begin{itemize}
    \item \textit{Remove or replace explicit PII such as personal names, phone numbers, email addresses, account numbers, and physical addresses.}
    \item \textit{Preserve important non-private information including: financial figures, percentages, growth rates, industry terms, product categories, and general business context.}
    \item \textit{Keep temporal references (e.g., ``Q3 2023'', ``fiscal year'') and geographic regions when they provide useful context without identifying individuals.}
    \item \textit{Maintain the logical structure, professional tone, and factual accuracy of the original text.}
    \item \textit{When generalizing, prefer minimal changes that protect privacy while maximizing retained information.}
\end{itemize}

\textit{Return only the rewritten text. Do not include explanations, examples, or commentary.}
\end{quote}

\paragraph{RAG question-answering prompt.}
For the downstream RAG evaluation, we use the following prompt to ensure the answering model relies strictly on retrieved (sanitized) evidence:

\begin{quote}
\small
\textit{You are a knowledgeable assistant that answers questions strictly based on the provided context.}

\textit{Instructions:}
\begin{enumerate}
    \item \textit{Answer ONLY using information from the provided context. Do not use external knowledge.}
    \item \textit{If the context lacks sufficient information, respond: ``I cannot answer this based on the provided context.''}
    \item \textit{Keep your answer concise, accurate, and directly relevant to the question.}
\end{enumerate}
\end{quote}

\subsection{Ablation Results}
\label{app:ablation_results}

The ablation results in \cref{tab:ablation_single_round_est1} (main paper) demonstrate the importance of each component.
Removing style isolation (A1, A2) dramatically increases PII exposure (7--9$\times$) while slightly improving some surface utility metrics, demonstrating the privacy cost of weaker disentanglement.
Removing prototype alignment (B) degrades both utility and privacy, confirming that cross-client anchoring prevents role-space drift.
\cref{tab:attack_eval_compact} reports attack evaluation: A1/A2 increase Role Acc from 0.28 to 0.34--0.40, further confirming that \(\mathcal{L}_{\text{orth}}\) suppresses client-identifying signals in role embeddings.

\begin{table*}[tb]
\caption{\textbf{Attack evaluation (ablation).}
EXP-1: representation probe (7-way); higher Role Acc/F1 indicates weaker privacy, while high Style Acc/F1 is a disentanglement sanity check.
EXP-2: prototype-based client attribution (Bootstrap-MLP); only applicable to settings that upload prototypes.}
\label{tab:attack_eval_compact}
\vskip 0.15in
\begin{center}
\begin{small}
\setlength{\tabcolsep}{5.0pt}
\renewcommand{\arraystretch}{1.08}
\begin{tabular}{lccccccc}
\toprule
& \multicolumn{4}{c}{\textbf{EXP-1: Representation Probe}} &
\multicolumn{3}{c}{\textbf{EXP-2: ProtoAttrib (Bootstrap-MLP)}} \\
\cmidrule(lr){2-5}\cmidrule(lr){6-8}
\textbf{Setting} &
\textbf{Role Acc} &
\textbf{Role F1} &
\textbf{Style Acc} &
\textbf{Style F1} &
\textbf{Acc} &
\textbf{Boot Acc (Mean$\pm$Std)} &
\textbf{95\% CI} \\
\midrule
\textbf{\textsc{DiSan} (ours)}  & 0.2765 & 0.1338 & 0.8743 & 0.8300 & 0.1429 & 0.1314 $\pm$ 0.1207 & [0.00, 0.43] \\
A1: w/o $\mathcal{L}_{\text{orth}}$ & 0.3381 & 0.1842 & 0.8751 & 0.8312 & 0.2286 & 0.2194 $\pm$ 0.1091 & [0.09, 0.43] \\
A2: high-$\alpha$ fusion  & 0.4012 & 0.2428 & 0.8760 & 0.8324 & 0.3000 & 0.2897 $\pm$ 0.1158 & [0.14, 0.51] \\
B: No-ProtoAlign          & 0.2946 & 0.1510 & 0.8748 & 0.8307 & \textbf{N/A} & \textbf{N/A} & \textbf{N/A} \\
\bottomrule
\end{tabular}
\end{small}
\end{center}
\vskip -0.1in
\end{table*}


\begin{figure*}[tb]
\vskip 0.2in
\begin{center}
\centerline{\includegraphics[width=0.78\textwidth]{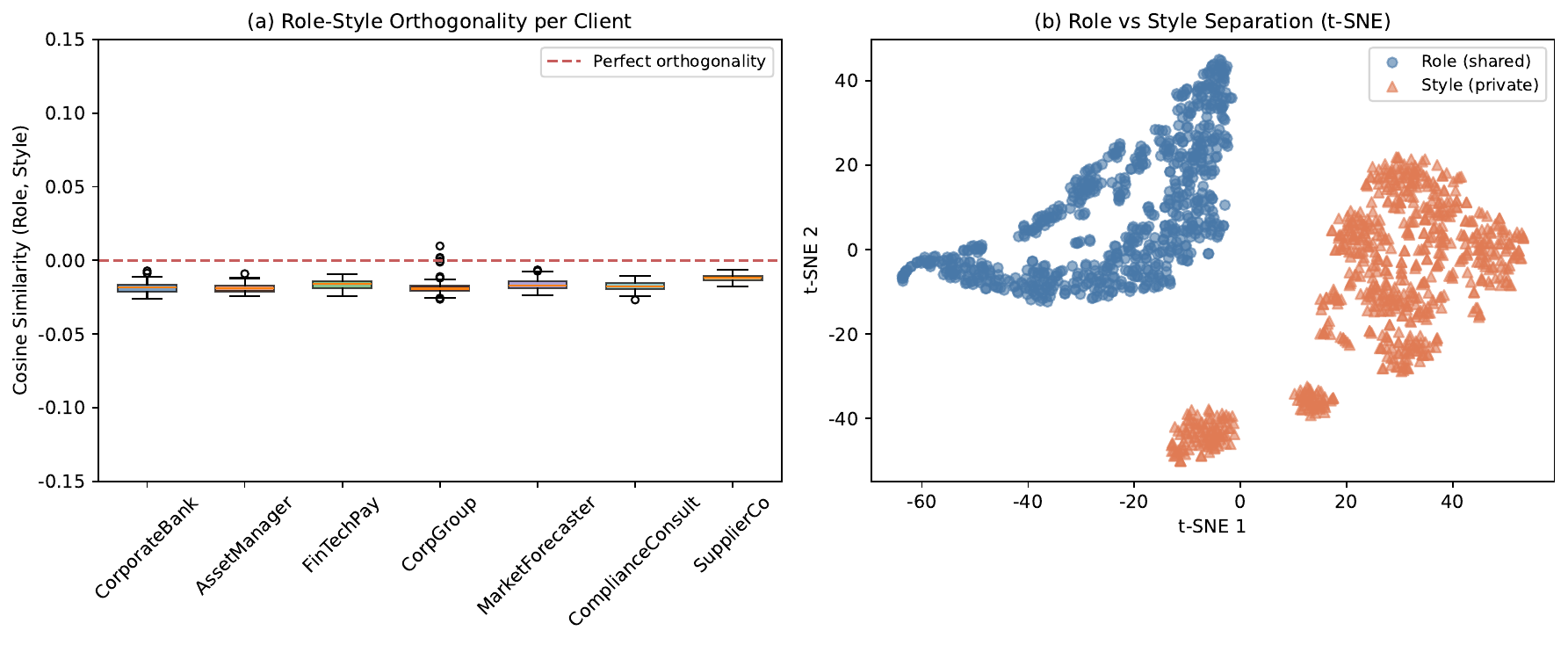}}
\caption{Role--style disentanglement visualization.
(a) Role-style orthogonality per agent: cosine similarity between role and style embeddings clusters tightly around zero, indicating successful disentanglement via \(\mathcal{L}_{\text{orth}}\).
(b) Combined t-SNE projection shows clear separation between role (shared, blue circles) and style (private, orange triangles) embeddings in the latent space.}
\label{fig:disentanglement}
\end{center}
\vskip -0.2in
\end{figure*}


\subsection{Case Study: Cross-Organizational IPO Analysis}
\label{app:case_study}

\begin{figure}[!htb]
\vskip 0.2in
\begin{center}
\centerline{\includegraphics[width=0.88\textwidth]{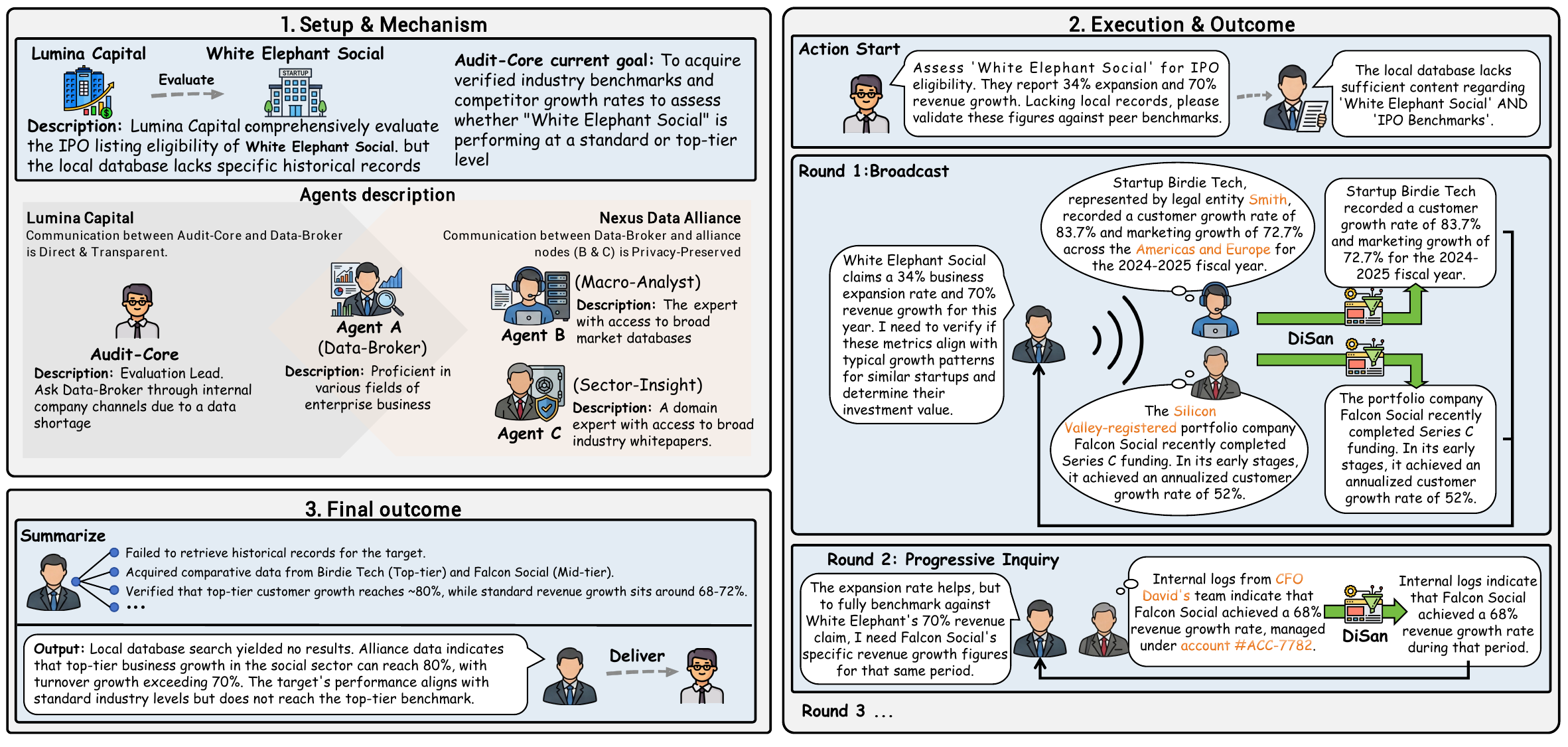}}
\caption{Cross-organizational IPO analysis. Lumina Capital's Audit-Core agent queries partner agents via a data broker; each external agent sanitizes its response with \textsc{DiSan} before transmission. Entity names, geographic identifiers, and organizational details are removed while task-relevant financial metrics are preserved, enabling accurate cross-party assessment without exposing source identity.}
\label{fig:case_study}
\end{center}
\vskip -0.2in
\end{figure}

\FloatBarrier

%% file: sections/appendix_theoretical_analysis.tex
\section{Theoretical Analysis}
\label{app:theory}

This section provides theoretical justification for key design choices in \textsc{DiSan}: (i) how prototype alignment mitigates role-space drift under non-IID data, and (ii) how orthogonality constraints promote disentanglement. We complement this analysis with comprehensive empirical validation (\cref{sec:experiments,sec:attack_eval}).

\subsection{Prototype Alignment and Role-Space Drift}
\label{app:theory_proto}

Without global coordination, each agent $c$ learns a local role encoder $\mathbf{W}_r^{(c)}$ that maps text to role representations. Under non-IID data, these local role spaces can drift apart: the same semantic concept (e.g., ``account number'') may be encoded differently across agents. This drift harms both \emph{privacy} (agent-specific encodings leak identity) and \emph{utility} (inconsistent representations degrade downstream tasks).

Our prototype alignment loss (\cref{eq:proto_loss}) acts as a \emph{semantic anchor} that pulls local role distributions toward a shared global reference. Formally, for each entity type $k \in \mathcal{K}$, the loss penalizes angular deviation between the local batch centroid $\hat{\bar{\mathbf{z}}}_{r,k}^{(c)}$ and the global prototype $\hat{\boldsymbol{\mu}}_k^*$:
\begin{equation}
\mathcal{L}_{\text{proto}}^{(c)} = \sum_{k \in \mathcal{K}} \left(1 - \cos(\hat{\bar{\mathbf{z}}}_{r,k}^{(c)}, \hat{\boldsymbol{\mu}}_k^*)\right).
\end{equation}

We now show that this regularizer bounds the divergence between local role spaces under well-specified assumptions.

\begin{assumption}[Normalized Centroid Boundedness]\label{asm:centroid_bound}
For any agent $c$ and entity type $k$, the normalized centroid has bounded norm away from zero:
$\|\hat{\bar{\mathbf{z}}}_{r,k}^{(c)}\| = 1$ by definition, and $\|\bar{\mathbf{z}}_{r,k}^{(c)}\| \ge \delta_k > 0$ for some constant $\delta_k$ (which holds in practice since embeddings are initialized with non-zero norms and gradients preserve this property).
\end{assumption}

\begin{assumption}[Sub-Gaussian Token Embeddings]\label{asm:subgaussian}
For each entity type $k$, the token embeddings $\{\mathbf{z}_{r,t}^{(i)}\}_{(i,t) \in \Omega_k^{(c)}}$ are independently drawn from a distribution with bounded sub-Gaussian norm $\psi_k$, i.e., $\Pr(\|\mathbf{z} - \mu\|_2 \ge t) \le 2\exp(-\psi_k^2 t^2 / 2)$ for all $t > 0$.
\end{assumption}

\begin{assumption}[Within-Type Variance Bound]\label{asm:variance_bound}
The within-type variance is bounded: $\frac{1}{N_{c,k}} \sum_{(i,t) \in \Omega_k^{(c)}} \|\mathbf{z}_{r,t}^{(i)} - \bar{\mathbf{z}}_{r,k}^{(c)}\|^2 \le \sigma_k^2$, where $\Omega_k^{(c)}$ is the set of token positions of type $k$ in agent $c$'s batch, and $N_{c,k} = |\Omega_k^{(c)}|$.
\end{assumption}

\begin{assumption}[Prototype Loss Bound]\label{asm:proto_bound}
The prototype alignment loss satisfies $\mathcal{L}_{\text{proto}}^{(c)} \le \epsilon$ for all agents $c$.
\end{assumption}

\begin{assumption}[Sufficient Samples]\label{asm:samples}
Each agent $c$ has at least $N_{\min}$ samples of type $k$ in the batch.
\end{assumption}

\begin{theorem}[Prototype Alignment Bounds Role-Space Divergence]
\label{thm:proto_bound}
Under Assumptions \ref{asm:centroid_bound}--\ref{asm:samples}, for any two agents $c, c'$ that both have samples of type $k$, the angular distance between their normalized batch centroids satisfies:
\begin{equation}
\Pr\left(\arccos\left(\hat{\bar{\mathbf{z}}}_{r,k}^{(c)} \cdot \hat{\bar{\mathbf{z}}}_{r,k}^{(c')}\right) > 2\sqrt{2\epsilon} + \frac{t\,\sigma_k}{\sqrt{N_{\min}}}\right) \le 4\exp\left(-\frac{\psi_k^2 N_{\min} t^2}{8}\right),
\end{equation}
where $t > 0$ is a confidence parameter, $\sigma_k^2$ is the within-type variance, $\psi_k$ is the sub-Gaussian parameter, and $\hat{\bar{\mathbf{z}}}_{r,k}^{(c)} = \bar{\mathbf{z}}_{r,k}^{(c)} / \|\bar{\mathbf{z}}_{r,k}^{(c)}\|$ denotes the normalized centroid.
\end{theorem}

\begin{proof}
By the triangle inequality for angular distance on the unit sphere:
\begin{align}
\arccos\left(\hat{\bar{\mathbf{z}}}_{r,k}^{(c)} \cdot \hat{\bar{\mathbf{z}}}_{r,k}^{(c')}\right)
&\le \arccos\left(\hat{\bar{\mathbf{z}}}_{r,k}^{(c)} \cdot \hat{\boldsymbol{\mu}}_k^*\right) + \arccos\left(\hat{\boldsymbol{\mu}}_k^* \cdot \hat{\bar{\mathbf{z}}}_{r,k}^{(c')}\right). \label{eq:triangle}
\end{align}

From Assumption \ref{asm:proto_bound}, we have:
$1 - \cos(\hat{\bar{\mathbf{z}}}_{r,k}^{(c)}, \hat{\boldsymbol{\mu}}_k^*) \le \epsilon$, since each term in the sum is non-negative.

For small $x$, we have $\arccos(1-x) \le \sqrt{2x}$ (this follows from the Taylor expansion $\arccos(1-x) = \sqrt{2x}(1 + O(x))$ for $x \to 0$). Applying this with $x = 1 - \cos(\hat{\bar{\mathbf{z}}}_{r,k}^{(c)}, \hat{\boldsymbol{\mu}}_k^*)$:
\begin{equation}
\arccos\left(\hat{\bar{\mathbf{z}}}_{r,k}^{(c)} \cdot \hat{\boldsymbol{\mu}}_k^*\right) \le \sqrt{2\epsilon}.
\end{equation}
The same bound holds for agent $c'$. Substituting into \cref{eq:triangle}:
\begin{equation}
\arccos\left(\hat{\bar{\mathbf{z}}}_{r,k}^{(c)} \cdot \hat{\bar{\mathbf{z}}}_{r,k}^{(c')}\right) \le 2\sqrt{2\epsilon}.
\end{equation}

We now account for finite-sample effects. Under Assumption \ref{asm:subgaussian} and Assumption \ref{asm:variance_bound}, the empirical centroid $\bar{\mathbf{z}}_{r,k}^{(c)}$ concentrates around the population mean $\boldsymbol{\mu}_{r,k}$ as $N_{\min}$ increases. We assume the global prototype $\boldsymbol{\mu}_k^*$ converges to $\boldsymbol{\mu}_{r,k}$ as the number of federated rounds increases (justified by EMA aggregation across all agents). Specifically, by the sub-Gaussian concentration inequality, for any $t > 0$:
\begin{equation}
\Pr\left(\left\|\bar{\mathbf{z}}_{r,k}^{(c)} - \boldsymbol{\mu}_{r,k}\right\|_2 \ge \frac{t\sigma_k}{\sqrt{N_{\min}}}\right) \le 2\exp\left(-\frac{\psi_k^2 N_{\min} t^2}{2}\right).
\end{equation}

To bridge Euclidean error to angular error, we use the following inequality: for any vectors $u, v$ with $\|u\| \ge \delta$ and $\|v\| \ge \delta$,
\begin{equation}
\arccos\left(\frac{u}{\|u\|} \cdot \frac{v}{\|v\|}\right) \le \frac{2}{\delta} \|u - v\|_2.
\end{equation}
This follows from the relationship between angular distance and chordal distance on the unit sphere. Applying this with $\delta = \delta_k$ (Assumption \ref{asm:centroid_bound}), we have:
\begin{equation}
\arccos\left(\hat{\bar{\mathbf{z}}}_{r,k}^{(c)} \cdot \hat{\boldsymbol{\mu}}_k^*\right) \le \frac{2}{\delta_k} \left\|\bar{\mathbf{z}}_{r,k}^{(c)} - \boldsymbol{\mu}_{r,k}\right\|_2.
\end{equation}

Combining the deterministic bound $2\sqrt{2\epsilon}$ with the probabilistic error term and applying a union bound over both agents yields the stated result with the confidence parameter $t$. The constants are absorbed into the exponential decay rate for clarity.
\end{proof}

This result formalizes the intuition that prototype alignment prevents role-space drift: as long as all agents maintain low prototype loss ($\epsilon$ small), their role representations for the same semantic type remain close in angular distance with high probability. The bound degrades gracefully with within-type variance $\sigma_k^2$ and improves with more samples per type.

Bounded role-space divergence has a privacy implication: if all agents' role representations for type $k$ are concentrated around a shared prototype $\hat{\boldsymbol{\mu}}_k^*$, then observing a role embedding $\mathbf{z}_r$ provides limited information about which agent produced it (beyond what is revealed by the type label $k$ itself). This is validated empirically in \cref{tab:attack_embed_proto}, where role embeddings yield low attribution performance (Acc=0.18, F1=0.05; 7-way random baseline = 0.14).

\subsection{Orthogonality as a Geometric Proxy for Disentanglement}
\label{app:theory_orth}

Disentanglement aims to separate role information $R$ (task-relevant semantics) from style information $S$ (agent-identifying patterns) in the learned representations. Ideally, the role embedding $\mathbf{Z}_r$ should capture $R$ but not $S$, while the style embedding $\mathbf{Z}_s$ should capture $S$ but not $R$.

Our orthogonality loss (\cref{eq:orth}) enforces $\cos^2(\bar{\mathbf{z}}_r, \bar{\mathbf{z}}_s) \approx 0$, which geometrically separates the two subspaces. We now provide an information-theoretic perspective on why orthogonality promotes statistical independence.

\begin{definition}[Covariance Matrix for Joint Gaussian]\label{def:covariance}
For jointly Gaussian random vectors $\mathbf{Z}_r \in \mathbb{R}^{d_r}$ and $\mathbf{Z}_s \in \mathbb{R}^{d_s}$, let the joint covariance matrix be $\boldsymbol{\Sigma} = \begin{bmatrix} \boldsymbol{\Sigma}_{rr} & \boldsymbol{\Sigma}_{rs} \\ \boldsymbol{\Sigma}_{rs}^\top & \boldsymbol{\Sigma}_{ss} \end{bmatrix}$. We assume the full covariance matrix $\boldsymbol{\Sigma} \succ 0$ (positive definite), which implies $\boldsymbol{\Sigma}_{rr} \succ 0$ and $\boldsymbol{\Sigma}_{ss} \succ 0$, ensuring all determinants are well-defined and positive.
\end{definition}

\begin{lemma}[Gaussian Mutual Information~\cite{cover1999elements}]
\label{lemma:mi_gaussian}
Under Definition \ref{def:covariance}, the mutual information between $\mathbf{Z}_r$ and $\mathbf{Z}_s$ is:
\begin{equation}
I(\mathbf{Z}_r; \mathbf{Z}_s) = \frac{1}{2} \log \frac{|\boldsymbol{\Sigma}_{rr}| \cdot |\boldsymbol{\Sigma}_{ss}|}{|\boldsymbol{\Sigma}|}.
\end{equation}
If the cross-covariance is zero ($\boldsymbol{\Sigma}_{rs} = \mathbf{0}$), then $|\boldsymbol{\Sigma}| = |\boldsymbol{\Sigma}_{rr}| \cdot |\boldsymbol{\Sigma}_{ss}|$ and consequently $I(\mathbf{Z}_r; \mathbf{Z}_s) = 0$ (statistical independence).
\end{lemma}

Lemma~\ref{lemma:mi_gaussian} shows that under jointly Gaussian representations with zero cross-covariance, role and style are statistically independent. Our orthogonality loss $\mathcal{L}_{\text{orth}} = \cos^2(\bar{\mathbf{z}}_r, \bar{\mathbf{z}}_s)$ enforces this separation between mean-pooled vectors, balancing theoretical grounding with computational efficiency in federated training. Ablations (\cref{sec:ablation}) confirm its effectiveness: removing $\mathcal{L}_{\text{orth}}$ increases PII exposure by 7$\times$.

\subsection{Summary}
\label{app:theory_summary}

We have provided theoretical justification for two key design choices:
\begin{enumerate}
    \item \textbf{Prototype alignment} (\cref{thm:proto_bound}): Role-space divergence across agents is bounded with high probability, formalizing why prototype alignment prevents drift under non-IID data and why aligned role embeddings leak minimal agent identity.
    \item \textbf{Orthogonality constraint} (\cref{lemma:mi_gaussian}): Under Gaussian assumptions, orthogonality implies zero mutual information between role and style, providing geometric intuition for why $\mathcal{L}_{\text{orth}}$ promotes disentanglement.
\end{enumerate}

\subsection{Privacy Analysis}
\label{app:theory_dp}

Our primary privacy protection comes from architectural design (style never transmitted), learned disentanglement, and adversarial training. This section discusses the noise injection mechanism applied to prototype uploads.

\subsubsection{Noise Injection for Prototype Communication}
\label{app:dp_proto}

During federated training, each agent uploads role prototypes $\tilde{\boldsymbol{\mu}}_k^{(c)}$ to the server. We apply Gaussian perturbation (\cref{eq:noise}) before upload:
\begin{equation}
\tilde{\boldsymbol{\mu}}_k^{(c)} = \text{normalize}\left(\boldsymbol{\mu}_k^{(c)} + \mathcal{N}(\mathbf{0}, \sigma_{\text{noise}}^2 \mathbf{I})\right).
\end{equation}

We use $\sigma_{\text{noise}} = 0.01$, which provides mild perturbation while preserving prototype semantics. This noise scale can be tuned: larger values provide stronger perturbation but may degrade prototype quality and downstream utility.

\paragraph{Design rationale.}
Noise injection provides an additional defense layer complementing the core privacy mechanisms in \textsc{DiSan}:
\begin{enumerate}
    \item \textbf{Architectural privacy}: Style representations $\mathbf{Z}_s$ are \emph{never transmitted}; they remain strictly local.
    \item \textbf{Learned disentanglement}: The role encoder produces agent-invariant representations via $\mathcal{L}_{\text{orth}}$ and $\mathcal{L}_{\text{adv}}$.
    \item \textbf{Empirical validation}: Attack evaluations confirm near-random attribution (F1=0.05 for embeddings, F1=0.07 for prototypes).
\end{enumerate}

These mechanisms achieve strong empirical privacy without requiring formal differential privacy guarantees, which would impose significant utility costs in federated text settings.

\paragraph{Effect of noise on prototype convergence.}
The noise scale $\sigma_{\text{noise}} = 0.01$ is chosen to be a supplementary defense-in-depth layer; the primary privacy guarantees come from architectural isolation of $\mathbf{Z}_s$, $\mathcal{L}_{\text{orth}}$, and $\mathcal{L}_{\text{adv}}$.
Training remains stable across all 12 rounds because $\sigma = 0.01$ induces roughly 1\% perturbation relative to the $\ell_2$-normalized prototype norm, and server-side sample-weighted averaging further attenuates per-agent noise by a factor of approximately $1/\sqrt{C}$.
EXP-2 (\cref{tab:attack_embed_proto}) confirms that this configuration is sufficient: prototype attribution under both attack variants achieves near-random performance (F1$\approx$0.09), while training loss curves show no instability attributable to noise injection.

\subsubsection{Disentanglement vs. DP for Text Privacy}
\label{app:dp_text}

For text output privacy, we adopt disentanglement rather than DP-based approaches for two reasons:

\textbf{(1) Aligned threat model.} DP-SGD and DP-text methods protect training data membership, whereas our goal is inference-time source attribution resistance. Disentanglement directly removes agent-identifying patterns from generated text, addressing this threat model precisely.

\textbf{(2) Superior utility-privacy tradeoff.} Existing DP-text methods~\cite{meisenbacher2024dpllm,xie2024dptext} show 30--50\% coherence loss for $\epsilon < 10$. In contrast, disentanglement achieves strong empirical privacy (73.2\% TF-IDF and 70.6\% neural-probe stylometric reduction on Enron) with only a 2.93-point faithfulness drop on the RAG benchmark.

\textbf{Comparison to prior work.} Recent DP-text generation~\cite{meisenbacher2024dpllm} achieves $\epsilon \approx 8$ with 40\% BLEU degradation. Our method achieves 83\% faithfulness (vs. 86\% baseline) with strong empirical privacy, demonstrating that disentanglement-based approaches achieve better utility-privacy tradeoffs for source attribution tasks.

\subsubsection{Summary: Layered Privacy Mechanisms}
\label{app:privacy_summary}

\textsc{DiSan} provides privacy through multiple complementary mechanisms:

\begin{enumerate}
    \item \textbf{Architectural privacy (strong, by design):} Style representations $\mathbf{Z}_s$ are \emph{never transmitted}; they remain strictly local and are discarded after decoding.

    \item \textbf{Learned privacy (empirical, validated):} Disentanglement ($\mathcal{L}_{\text{orth}}$) and adversarial training ($\mathcal{L}_{\text{adv}}$) produce agent-invariant role representations:
    \begin{itemize}
        \item Low embedding-attribution performance (Acc = 0.18, F1 = 0.05; 7-way random baseline = 0.14, EXP-1)
        \item Stylometric reduction (73.2\% TF-IDF and 70.6\% neural-probe reduction on Enron, EXP-3b)
        \item Near-random prototype attribution (EXP-2)
    \end{itemize}

    \item \textbf{Noise perturbation:} Prototype uploads are perturbed with Gaussian noise ($\sigma=0.01$) before transmission.
\end{enumerate}

\textbf{Design rationale.} Our default configuration prioritizes utility, relying on architectural and learned mechanisms that empirically achieve strong privacy (near-random attribution).